\documentclass{article} 
\usepackage{arxiv,times}

\usepackage{url}
\usepackage[breaklinks]{hyperref}
\usepackage{url}
\usepackage{amssymb}
\usepackage{amsthm}
\usepackage{amsmath}
\usepackage{thmtools}
\usepackage{thm-restate}
\usepackage{cleveref}
\usepackage{enumitem}
\usepackage{verbatim}

\usepackage[pdftex]{graphicx}
\usepackage{caption}
\usepackage{subcaption}
\usepackage{float}
\usepackage{tikz}

\newcommand{\norm}[1]{\left\|#1\right\|}
\newcommand{\parens}[1]{\left(#1\right)}
\newcommand{\braces}[1]{\left\{#1\right\}}
\newcommand{\brackets}[1]{\left[#1\right]}
\newcommand{\abs}[1]{\left|#1\right|}
\newcommand{\RR}{\mathbb{R}}
\newcommand{\EE}{\mathbb{E}}
\newcommand{\PP}{\mathbb{P}}
\newcommand{\NN}{\mathbb{N}}
\newcommand{\ind}[1]{\mathbb{I}\left\{#1\right\}}

\def\deanonymize{1}

\declaretheorem[name=Theorem, numberwithin=section]{theo}
\declaretheorem[name=Lemma, sibling=theo]{lem}

\declaretheorem[name=Corollary, sibling=theo]{cor}
\declaretheorem[name=Definition, sibling=theo]{defi}

\declaretheorem[name=Assumption, sibling=theo]{ass}

\declaretheorem[name=Observation, sibling=theo]{obs}

\usepackage{color}

\title{Spectral Bias Outside the Training Set for Deep Networks in the Kernel Regime}

\author{Benjamin Bowman\\
UCLA Department of Mathematics\\
\texttt{benbowman314@math.ucla.edu}
\And
Guido Mont\'ufar\\ 
UCLA Departments of Mathematics and Statistics and MPI MIS\\ 
\texttt{montufar@math.ucla.edu} \\
}

\begin{document}

\maketitle

\begin{abstract}
\noindent
We provide quantitative bounds measuring the $L^2$ difference in function space between the trajectory of a finite-width network trained on finitely many samples from the idealized kernel dynamics of infinite width and infinite data.  An implication of the bounds is that the network is biased to learn the top eigenfunctions of the Neural Tangent Kernel not just on the training set but over the entire input space.  This bias depends on the model architecture and input distribution alone and thus does not depend on the target function which does not need to be in the RKHS of the kernel.  The result is valid for deep architectures with fully connected, convolutional, and residual layers.  Furthermore the width does not need to grow polynomially with the number of samples in order to obtain high probability bounds up to a stopping time.  The proof exploits the low-effective-rank property of the Fisher Information Matrix at initialization, which implies a low effective dimension of the model (far smaller than the number of parameters).  We conclude that local capacity control from the low effective rank of the Fisher Information Matrix is still underexplored theoretically.
\end{abstract}

\section{Introduction}
Training heavily overparameterized networks via gradient based optimization has become standard operating procedure in deep learning.  Overparameterized networks are able to interpolate arbitrary labels both in principle and in practice \citep{zhang2017understanding}, rendering classical PAC learning theory insufficient to explain the generalization of networks within this modality.  The high capacity of modern networks ensures that there are both good and bad empirical risk minimizers.  Miraculously the network preferentially chooses the good solutions and sidesteps those that are unfavorable, posing a challenge and opportunity to today's researchers. 
\par
The success of overparameterized networks has prompted the theoretical community to search for more subtle forms of capacity control \citep{neyshabur2015search,neyshabur2017geometry,NIPS2017_58191d2a}.  The contemporary point-of-view is that the data distribution, model parameterization, and optimization algorithm are all relevant in limiting complexity.  This has led to a variety of efforts to characterize the properties that networks and related models are biased towards when optimized via gradient descent.  Examples include max-margin bias for classification problems \citep{soudry2018implicit, 
telgarskynonseparable, Nacson2019ConvergenceOG,NEURIPS2018_0e98aeeb}, minimum nuclear norm bias for matrix factorization \citep{NIPS2017_58191d2a,pmlr-v75-li18a,NEURIPS2018_0e98aeeb}, and minimum RKHS norm bias in the kernel regime \citep{generalizationinducedbyinit}.
\par
Empirically it is known that neural networks tend to learn low Fourier frequencies first and add higher frequencies only later in training \citep{rahaman2019spectral, xu2019training, yang2022overcoming}, the phenomenon that has been titled ``Spectral Bias" or the ``Frequency Principle".  Theoretical justifications of this have been proposed by studying networks in the kernel regime.  For shallow univariate ReLU networks \cite{basri2019convergence, basri2020frequency} demonstrate that the dominant eigenfunctions of the Neural Tangent Kernel (NTK) \citep{jacot2020neural} correspond to the low Fourier frequencies for the uniform distribution and more generally to smoother components for nonuniform distributions.  This echos the results by \citet{williams2019gradient} and \citet{jin2021implicit} that show that univariate ReLU networks in the kernel regime are biased towards smooth interpolants.  Abstracting away from Fourier frequencies, ``Spectral Bias" can be interpreted more broadly to mean bias towards learning the top eigenfunctions of the Neural Tangent Kernel.  By looking at empirical approximations to the eigenfunctions, spectral bias was demonstrated to hold on the training set by \citet{arora2019finegrained}, \citet{basri2020frequency}, and \citet{caounderstanding}.  A recent work by \citet{bowman2022implicit} was able to demonstrate that spectral bias holds off the training set for shallow feedforward networks when the network is underparameterized.  In the present work we exploit the low-effective-rank property of the Fisher Information Matrix and are able to demonstrate that spectral bias holds outside the training set without the underparameterization requirement.  In fact the number of samples can be on the same order as the width of the network.  Furthermore, by leveraging a recent work by \citet{liuonlinearity} bounding the Hessian of wide networks, our result permits deep networks with fully connected, convolutional, and residual layers.  Consequently we are able to conclude that spectral bias holds for more realistic sample complexities and diverse architectures.

\subsection{Our Contributions}
\begin{itemize}[leftmargin=*] 
    \item We provide quantitative bounds measuring the $L^2$ difference in function space between the trajectory of a finite-width network trained on finitely many samples from the idealized kernel dynamics of infinite width and infinite data (see Theorem~\ref{thm:main} and Corollary~\ref{cor:maincor}). 
    \item As an implication of these bounds, eigenfunctions of the NTK integral operator (not just their empirical approximations) are learned at rates corresponding to their eigenvalues (see Corollary~\ref{cor:maincor} and Observation~\ref{obs:spectralbias}). 
    \item We demonstrate that the network will inherit the bias of the kernel at the beginning of training even when the width only grows linearly with the number of samples (see Observation~\ref{obs:linearover}).
\end{itemize}

\subsection{Related Work}
\paragraph{NTK Convergence Results} 
The NTK was introduced by \citet{jacot2020neural} while almost concurrently \citet{du2018gradient} used it implicitly to prove a global convergence guarantee for gradient descent applied to a shallow ReLU network.  These two highly charismatic works led to a flurry of subsequent works, of which we can only hope to provide a partial list.  Global convergence for arbitrary labels was addressed in a series of works \citep{du2018gradient,du2019gradient,oymak2019moderate,allenzhu2019convergence,qyunhpyramidal,nguyenrelu,zou2018stochastic,zou2019improved}.  For arbitrary labels to our knowledge all works require the network width to either grow polynomially with the number of samples $n$ or the inverse desired accuracy $\epsilon^{-1}$.  If one assumes the target function aligns with the NTK model, for shallow networks this can be reduced to polylogarithmic width for the logistic loss \citep{Ji2020Polylogarithmic} or linear width for the squared loss \citep{Weinan2020ACA,su2019learning,bowman2022implicit}.

\paragraph{Spectrum of the NTK/Hessian and Generalization}
The fact that the NTK tends to have a small number of large outlier eigenvalues has been observed in many works (e.g.\ \citealt{arora2019finegrained,oymak2020generalization,li2019gradient}).  \citet{papyantraces} demonstrated that for classification problems the logit gradients cluster within classes, which produces outliers in the spectra of the NTK and the Hessian of the loss.  There have been a series of works analyzing the NTK/Hessian spectrum theoretically using random matrix theory and other tools (e.g.\ \citealt{karakida_pathological_2021,NEURIPS2018_18bb68e2,pmlr-v70-pennington17a, 10.5555/3495724.3496370,https://doi.org/10.48550/arxiv.1907.10599}).  Recently the spectrum of the NTK integral operator for ReLU networks has been shown to asymptotically follow a power law \citep{yarotskyasymptotics}.  \citet{arora2019finegrained} provided a generalization bound that is effective when the labels align with the top eigenvectors of the NTK.  \citet{oymak2020generalization} were able to use the low effective rank of the NTK to obtain generalization bounds, and \citet{li2019gradient} used the same property to demonstrate robustness to label noise.  The low effective rank of the Hessian has also been incorporated into PAC-Bayes bounds, most recently by \citet{Yang2021DoesTD}.  Interestingly, the notion of the effective dimension they define is essentially the same quantity we use to bound the model complexity of the network's linearization. 

\paragraph{NTK Eigenvector and Eigenfunction Convergence Rates}
\citet{https://doi.org/10.48550/arxiv.2010.08153} explicitly tracked the dynamics of the infinite width shallow model in the Fourier domain.  \citet{arora2019finegrained} demonstrated that when training the hidden layer of a shallow ReLU network, the residual error on the training set projected along eigenvectors of the NTK Gram matrix decays linearly at rates corresponding to the eigenvalues.  \citet{caounderstanding} proved a similar statement for training both layers, and \citet{basri2020frequency} proved the analogous statement for a deep fully connected ReLU network where the first and last layer are fixed.  Our result can be viewed as the corresponding statement for the test residual instead of the empirical residual: projections of the test residual along eigen\textit{functions} of the NTK \textit{integral operator} are learned at rates corresponding to their eigenvalues.  This was shown in a recent work \citep{bowman2022implicit} for shallow fully connected networks that are underparameterized.  By contrast our result does not require the network to be underparameterized, and holds for deep networks with fully connected, convolutional, and residual layers.  We view our fundamental contribution as demonstrating that spectral bias holds with more realistic sample complexities and in considerable generality with respect to model architecture.

\section{Preliminaries}
\label{sec:preliminaries}

\subsection{Notation}
Vectors $v \in \RR^k$ will be column vectors by default.  We will let $\langle \bullet, \bullet \rangle$ and $\norm{\bullet}_2$ denote the Euclidean inner product and norm.  We define $\langle \bullet, \bullet \rangle_{\RR^n} = \frac{1}{n} \langle \bullet, \bullet \rangle$ and $\norm{\bullet}_{\RR^n} := \sqrt{\langle \bullet, \bullet \rangle_{\RR^n}}$ to be the normalized Euclidean inner product and norm.  The notation $\overline{B}(v, r) := \{ w : \norm{w - v}_2 \leq r\}$ will denote the \textit{closed} Euclidean ball centered at $v$ of radius $r$.  $\norm{A}_{op} := \sup_{\norm{v}_2 = 1} \norm{Av}_2$ will denote the operator norm for matrices.  For a symmetric matrix $A \in \RR^{k \times k}$, $\lambda_i(A)$ denotes its $i$-th largest eigenvalue, i.e.\ $\lambda_1(A) \geq \lambda_2(A) \geq \cdots \geq \lambda_k(A)$.  For a set $A$ we will let $|A|$ denote its cardinality.  For a natural number $k \geq 1$, we will let $[k] := \{1, \ldots, k\}$.  We will let $L^p(X, \nu)$ denote the $L^p$ space over domain $X$ with measure $\nu$.  We will denote the inner product associated with $L^2(X, \nu)$ as $\langle \bullet, \bullet \rangle_\nu$.  We will use the standard big $O$ and $\Omega$ notation with $\Tilde{O}$ and $\Tilde{\Omega}$ hiding logarithmic terms.

\subsection{NTK Dynamics}
Let $f(x; \theta)$ be our scalar-valued neural network model taking inputs $x \in X \subset \RR^d$ parameterized by $\theta \in \RR^p$.  For now we will not specify a specific architecture.  Our training data will be $n$ input-label pairs $\{(x_1, y_1), \ldots, (x_n, y_n)\} \subset \RR^d \times \RR$ where we assume that the labels $y_i$ are generated from a fixed scalar-valued target function $f^*$, namely $f^*(x_i) = y_i$.  We will let $y \in \RR^n$ denote the label vector $y = (y_1, \ldots, y_n)^T$.  Let $\hat{r}(\theta) \in \RR^n$ denote the vector that measures the residual error on the training set, whose $i$-th entry is $\hat{r}(\theta)_i := f(x_i; \theta) - y_i$.
We will optimize the squared loss
\[ \Phi(\theta) := \frac{1}{2n} \norm{\hat{r}(\theta)}_2^2 = \frac{1}{2} \norm{\hat{r}(\theta)}_{\RR^n}^2 \]
via gradient flow
\[ \partial_t \theta_t = - \partial_{\theta} \Phi(\theta), \]
which is the continuous time analog of gradient descent.  For conciseness we will denote $\hat{r}(\theta_t)$ by $\hat{r}_t$ and let $r_t(x) := f(x; \theta_t) - f^*(x)$ denote the residual for an arbitrary input $x$ not necessarily in the training set.  We may also write $r(x; \theta) := f(x; \theta) - f^*(x)$ for the residual for an arbitrary $\theta$.
\par
We recall some key definitions and facts about the NTK.  For a comprehensive introduction we refer the reader to \citet{jacot2020neural}.
We recall the definition of the analytical NTK 
\[ K^\infty(x, x') := \EE_{\theta_0 \sim \mu}\brackets{\langle \nabla_\theta f(x; \theta_0), \nabla_\theta f(x'; \theta_0) \rangle}, \]
where the expectation is taken over the parameter initialization $\theta_0 \sim \mu$.  The kernel $K^\infty$ induces an integral operator $T_{K^\infty} : L^2(X, \rho) \rightarrow L^2(X, \rho)$ 
\begin{equation}\label{eq:tkinftydef}
T_{K^\infty} g(x) := \int_X K^\infty(x, s) g(s) d\rho(s) ,     
\end{equation}
where $X$ is our input space and $\rho$ is the input distribution.  We assume our training inputs $x_1, \ldots, x_n$ are i.i.d.\ samples from $\rho$.  More generally, for a continuous kernel $K(x, x')$ we define $T_K : L^2(X, \rho) \rightarrow L^2(X, \rho)$
\begin{equation}\label{eq:tkdefinition}
T_K g(x) := \int_X K(x, s) g(s) d\rho(s).    
\end{equation}
Returning back to $K^\infty$, by Mercer's theorem we have the decomposition
\[ K^\infty(x, x') = \sum_{i = 1}^\infty \sigma_i \phi_i(x) \phi_i(x'), \]
where $\{\phi_i\}$ is an orthonormal basis for $L^2(X, \rho)$ and $\{\sigma_i\}$ is a nonincreasing sequence of positive values.  We will see that the bias at the beginning of training within our framework can be described entirely through the operator $T_{K^\infty}$ and its eigenfunctions.  We note that $T_{K^\infty}$ depends only on the model architecture, parameter initialization distribution $\mu$, and input distribution $\rho$.  
The training data sample $x_1, \ldots, x_n$ introduces a discretization of the operator $T_{K^\infty}$
\begin{equation}\label{eq:tndefinition}
T_n g(x) := \frac{1}{n} \sum_{i = 1}^n K^\infty(x, x_i) g(x_i) = \int_X K^\infty(x, s) g(s) d\widehat{\rho}(s) , 
\end{equation}
where $\widehat{\rho} = \frac{1}{n} \sum_{i = 1}^n \delta_{x_i}$ is the empirical measure.
We now introduce the time-dependent NTK 
\[ K_t(x, x') := \langle \nabla_\theta f(x;\theta_t), \nabla_\theta f(x'; \theta_t) \rangle \]
with the associated time-dependent operator $T_n^t$
\begin{equation}\label{eq:tntdefinition}
T_n^t g(x) := \frac{1}{n} \sum_{i = 1}^n K_t(x, x_i) g(x_i) = \int_X K_t(x, s) g(s) d\widehat{\rho}(s).    
\end{equation}
The update rule for the residual $r_t$ under gradient flow is given by
\[ \partial_t r_t(x) = - \frac{1}{n} \sum_{i = 1}^n K_t(x, x_i) r_t(x_i) = - T_n^t r_t. \]
Speaking loosely, as the network width tends to infinity the time-dependent NTK $K_t(x, x')$ becomes constant so that $K_t(x, x') = K^\infty(x, x')$ uniformly in $t$.  If $K_t = K^\infty$ then we have the operator equality $T_n^t = T_n$.  Similarly, heuristically as $n \rightarrow \infty$ we have $T_n \rightarrow T_{K^\infty}$.  Thus in the idealized infinite width, infinite data limit the update rule becomes
\[ \partial_t r_t = -T_{K^\infty} r_t ,\]
which has the solution $r_t = \exp(-T_{K^\infty} t) r_0$ which is defined via its projections
\[ \langle r_t, \phi_i \rangle_\rho = \exp(-\sigma_i t) \langle r_0, \phi_i \rangle_\rho. \]
Thus in this idealized setting the network learns eigenfunctions $\phi_i$ at rates determined by their eigenvalues $\sigma_i$.  The dependence of the convergence rate on the magnitude of $\sigma_i$ is particularly relevant as the NTK tends to have a very skewed spectrum.  We can estimate the spectrum of $K^\infty$ by randomly initializing a network and computing the Gram matrix $(G_0)_{i,j} := K_0(x_i, x_j)$.  
In Figure~\ref{fig:ntk_spec} we plot the spectrum of the NTK Gram Matrix $(G_0)_{i,j} := K_0(x_i, x_j)$ at initialization.  We observe a small number of outlier eigenvalues of large magnitude followed by a long tail of small eigenvalues.  This phenomenon has appeared in many works (e.g.\ \citealt{arora2019finegrained, oymak2020generalization,li2019gradient}).  For ReLU networks the spectrum is known to asymptotically follow a power law $\sigma_i \sim \Lambda i^{-\nu}$ \citep{yarotskyasymptotics}.  The goal of this work is to quantify the extent to which a finite-width network trained on finitely many samples behaves like the idealized kernel dynamics $r_t = \exp(-T_{K^\infty} t) r_0$ corresponding to infinite width and infinite data.  
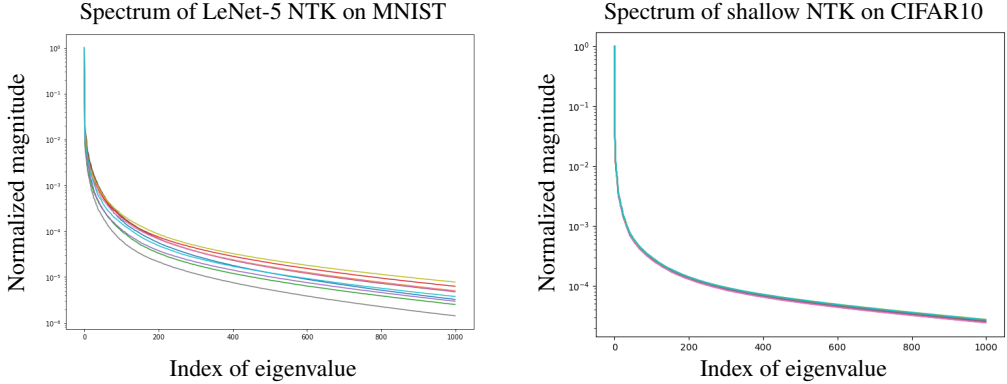
\begin{figure}
      \begin{subfigure}[b]{0.5\textwidth}
         \centering
         \begin{tikzpicture}
         \node at (0,0) {\includegraphics[width=\textwidth]{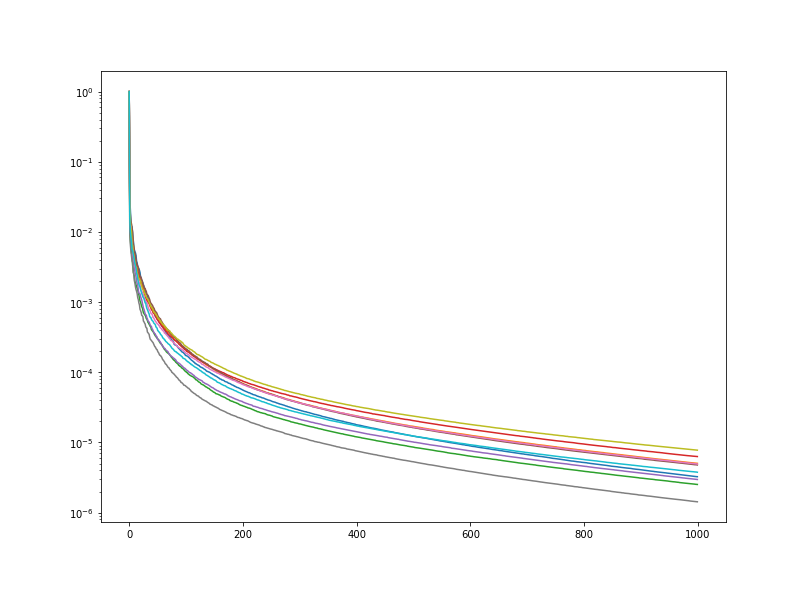}}; 
         \node at (0,-2.5){\small Index of eigenvalue};
         \node at (-3.25,0){\rotatebox{90}{\small Normalized magnitude}}; 
         \node at (0,2.25) {\small Spectrum of LeNet-5 NTK on MNIST}; 
         \end{tikzpicture}
    \end{subfigure}
     \hfill
     \begin{subfigure}[b]{0.5\textwidth}
         \centering
         \begin{tikzpicture}
         \node at (0,0) {\includegraphics[width=\textwidth]{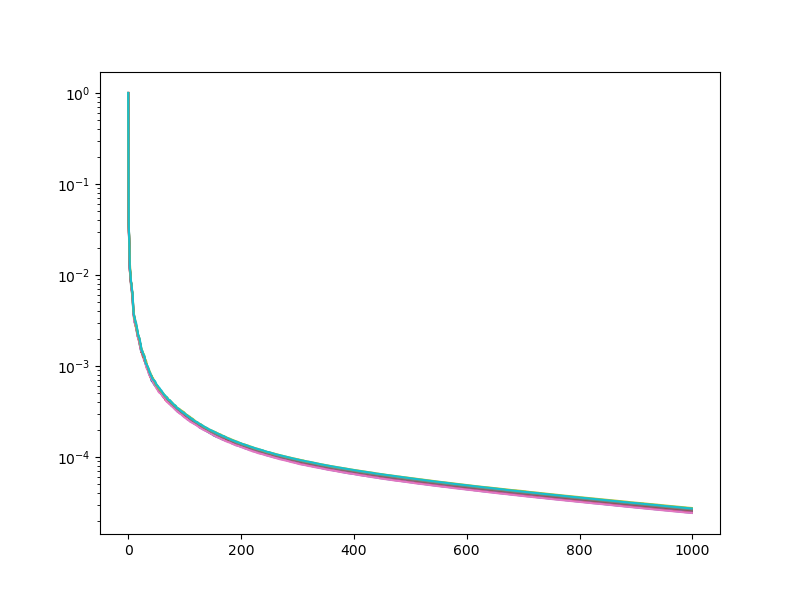}}; 
         \node at (0,-2.5){\small Index of eigenvalue};
         \node at (-3.25,0){\rotatebox{90}{\small Normalized magnitude}}; 
        \node at (0,2.25) {\small Spectrum of shallow NTK on CIFAR10};          
         \end{tikzpicture}
     \end{subfigure}
        \caption{We plot the NTK spectrum on MNIST and CIFAR10 for two networks using 10 random parameter initializations and data batches.  In both plots the x-axis represents the eigenvalue index $k$ (linear scale) and the y-axis 
        the normalized eigenvalue $\lambda_k / \lambda_1$ magnitude (log scale).  To avoid numerical issues, we compute the NTK on  
        a batch of size 2000 and plot the first 1000 eigenvalues.
        The left plot computed the NTK corresponding to the logit of class 0 for LeNet-5 on MNIST.  The right plot is for a shallow fully-connected softplus network with 4000 hidden units on CIFAR10.}
        \label{fig:ntk_spec}
\end{figure}

\subsection{Applicable Architectures}\label{sec:architecture}
We now specify an architecture for our model $f(x; \theta)$.  We consider deep networks of the form
\begin{align*}
\alpha^{(0)} &:= x ,\\
\alpha^{(l)} &:= \psi_l(\theta^{(l)}, \alpha^{(l - 1)}), \quad l \in [L], \\
f(x; \theta) &:= \frac{1}{\sqrt{m_{L}}} v^T \alpha^{(L)},
\end{align*}
where each $\psi_l(\theta^{(l)}, \bullet) : \RR^{m_{l - 1}} \rightarrow \RR^{m_l}$ is a vector-valued function parameterized by $\theta^{(l)} \in \RR^{p_l}$ and $v \in \RR^{m_L}$.  We define $\theta^{(L + 1)} := v$ and set $\theta := ((\theta^{(1)})^T, \ldots, (\theta^{(L + 1)})^T)^T$ to be the collection of all parameters.  We assume each layer mapping $\psi_l$ has one of the following forms: 
\begin{align*}
\text{Fully Connected}&: \psi_l(\theta^{(l)}, \alpha^{(l - 1)}) = \omega\parens{\frac{1}{\sqrt{m_{l - 1}}} W^{(l)} \alpha^{(l - 1)}}  \\
\text{Convolutional}&:  \psi_l(\theta^{(l)}, \alpha^{(l - 1)}) = \omega\parens{\frac{1}{\sqrt{m_{{l - 1}}}} W^{(l)} * \alpha^{(l - 1)}} \\
\text{Residual}&: \psi_l(\theta^{(l)}, \alpha^{(l - 1)}) = \omega\parens{\frac{1}{\sqrt{m_{l - 1}}} W^{(l)} \alpha^{(l - 1)}} + \alpha^{(l - 1)}  
\end{align*}
Here $\theta^{(l)} = vec(W^{(l)})$ and $\omega$ is a twice continuously differentiable function such that $\omega$ and $\omega'$ are Lipschitz.  All parameters of the network will be trained as in practice.  For feedforward and residual layers 
$W^{(l)} \in \RR^{m_l \times m_{l - 1}}$ is a matrix.  For the case of convolutional layers $W^{(l)} \in \RR^{K \times m_{l} \times m_{l - 1}}$ is an order-3 tensor with filter size $K$.  The precise definition of the convolution $*$ is offered in the appendix.  We will let $m = \min_{l} m_{l}$ denote the minimum width of the network.   We will assume that $\max_{l} \frac{m_l}{m} = O(1)$.  The input dimension $d := m_0$, the depth $L$, and the filter sizes $K$ of convolutional layers will be treated as constant.  The depth $L$ being constant is essential for NTK convergence; see \citet{Hanin2020Finite} for an explanation of failure modes whenever depth is nonconstant. 
\par
We will now discuss our initialization scheme.  We will perform the antisymmetric initialization trick introduced by \citet{generalizationinducedbyinit} so that the model is identically zero at initialization $f(\bullet; \theta_0) \equiv 0$.  Let $f(x; \theta)$ be any neural network of the form described above.  Then let $\tilde{\theta} = \left[ \begin{smallmatrix}
\theta\\
\theta'
\end{smallmatrix} \right]$
where $\theta, \theta' \in \RR^p$.
We then define
\[ f_{ASI}(x; \tilde{\theta}) := \frac{1}{\sqrt{2}} f(x; \theta) - \frac{1}{\sqrt{2}} f(x; \theta') \]
which takes the difference of two rescaled copies of our original model $f(x; \theta)$ with parameters $\theta$ and $\theta'$ that are optimized freely.  The antisymmetric initialization trick initializes $\theta_0 \sim N(0, I)$ then sets
$\tilde{\theta}_0 = \left[\begin{smallmatrix}
\theta_0\\
\theta_0
\end{smallmatrix}\right]$. 
We then optimize the model $f_{ASI}$ starting from the initialization $\tilde{\theta}_0$.  This trick simultaneously ensures that the model is identically zero at initialization without changing the NTK at initialization \citep{generalizationinducedbyinit}.  For ease of notation we will simply assume from now on that $f(x; \theta) = f_{ASI}(x; \theta)$ and not write the subscript $ASI$.

\section{Main Results}
\label{sec:main-results}
Before stating our main result, we enumerate our key assumptions for the sake of clarity, assumed to hold throughout. Detailed proofs are deferred to the appendix. 
\begin{ass}\label{ass:act}
The activation $\omega$ is twice continuously differentiable and $\omega$ and 
$\omega'$ are Lipschitz.
\end{ass}
\begin{ass}\label{ass:mercercond}
The input domain $X$ is compact with strictly positive Borel measure $\rho$.
\end{ass}
\begin{ass}\label{ass:boundedtarget}
The target function $f^*$ satisfies $\norm{f^*}_{L^\infty(X, \rho)} = O(1)$.
\end{ass}
\begin{ass}\label{ass:antisymmetric}
We use the antisymmetric initilization trick so that $f(\bullet; \theta_0) \equiv 0$.  
\end{ass}
\noindent
Most activation functions except for ReLU satisfy Assumption~\ref{ass:act}, such as Softplus $\omega(x) = \ln(1 + e^x)$, Sigmoid $\omega(x) = \tfrac{1}{1 + e^{-x}}$, and Tanh $\omega(x) = \tfrac{e^{x} - e^{-x}}{e^x + e^{-x}}$.  Assumption \ref{ass:mercercond} is a sufficient condition for Mercer's Theorem to hold.  While Mercer's theorem is often assumed to hold implicitly, we prefer to make this assumption explicit.  Assumption~\ref{ass:boundedtarget} simply means the target function is bounded.  We believe the antisymmetric initialization specified in Assumption~\ref{ass:antisymmetric} is not strictly necessary but it greatly simplifies the proofs and associated bounds.  To sidestep \ref{ass:antisymmetric} one would utilize high probability bounds on the magnitude $|f(x; \theta_0)|$ at initialization.  
In the following results $f(x; \theta)$ will be any of the architectures discussed in Section~\ref{sec:architecture}.  We are now ready to introduce the main result.
\begin{restatable}{theo}{main}
\label{thm:main}
Let $T \geq 1, \epsilon > 0$.  Let $K(x, x')$ be a fixed continuous, symmetric, positive definite kernel.  For $k \in \NN$ let $P_k : L^2(X, \rho) \rightarrow L^2(X, \rho)$ denote the orthogonal projection onto the span of the top $k$ eigenfunctions of the operator $T_K$ defined in Equation~\eqref{eq:tkdefinition}.  Let $\sigma_k > 0$ denote the $k$-th eigenvalue of $T_K$.  Then $m = \tilde{\Omega}(T^4 / \epsilon^2)$ and $n = \tilde{\Omega}(T^2 / \epsilon^2)$ suffices to ensure with probability at least $1 - O(mn)\exp(-\Omega(\log^2(m))$ over the parameter initilization $\theta_0$ and the training samples $x_1, \ldots, x_n$ that for all $t \leq T$ and $k \in \NN$
\[\norm{P_k(r_t - \exp(-T_{K}t)r_0)}_{L^2(X, \rho)}^2 
\leq \brackets{\frac{1 - \exp(-\sigma_k t)}{\sigma_k}}^2 \cdot \brackets{4 \norm{f^*}_{\infty}^2 \norm{K - K_0}_{L^2(X^2, \rho \otimes \rho)}^2 + \epsilon} \]
and
\[\norm{r_t - \exp(-T_{K}t)r_0}_{L^2(X, \rho)}^2 \leq t^2 \cdot \brackets{4 \norm{f^*}_\infty^2 \norm{K - K_0}_{L^2(X^2, \rho \otimes \rho)}^2 + \epsilon}.\]
\end{restatable} 
\subsection{Interpretation and Consequences}
Theorem \ref{thm:main} compares the dynamics of the residual $r_t(x) := f(x; \theta_t) - f^*(x)$ of our finite-width model trained on finitely many samples to the idealized dynamics of a kernel method $\exp(-T_K t)r_0$ with infinite data.  We recall that if $\phi_i$ is an eigenfunction of $T_K$ with eigenvalue $\sigma_i$ then $\langle \exp(-T_K t) r_0, \phi_i \rangle_\rho = \exp(-\sigma_i t) \langle r_0, \phi_i \rangle_\rho$.  Thus the term $\exp(-T_K t) r_0$ learns the projection along eigenfunction $\phi_i$ linearly at rate $\sigma_i$.  Whenever the NTK at initialization $K_0$ concentrates around $K$, the residual $r_t$ will inherit this bias of the kernel dynamics $\exp(-T_K t) r_0$.  Furthermore, the bound for the projected difference $\norm{P_k(r_t - \exp(-T_{K}t)r_0)}_{L^2(X, \rho)}^2$ is smaller whenever
$\sigma_k$ is large.  Therefore the bias appears more pronounced along eigendirections with large eigenvalues.
\paragraph{Consequences for the special case $K = K^\infty$} 
In the infinite width limit, we have that $K_0$ approaches $K^\infty$ for general architectures \citep{gregyangtp2}.  For fixed $x, x'$, by concentration results the typical rate of convergence is $|K_0(x, x') - K^\infty(x, x')| = \tilde{O}(1/\sqrt{m})$ with high probability \citep{du2018gradient, du2019gradient, huang2019dynamics}.  
Bounds that hold uniformly over $x, x'$ of the same rate were provided by \citet{bowman2022implicit} and \citet{buchanan2021deep}.  
A more pessimistic estimate of $1/m^{1/4}$ is provided by \citet{aroraexact}. Even if the rate is $1/m^{1/4}$, we have that $m = \tilde{\Omega}(\epsilon^{-2})$ is strong enough to ensure that $|K_0(x, x') - K^\infty(x, x')| \leq \epsilon^{1/2}$.  Given these results, it is reasonable to make the following assumption for the architectures we consider (see Appendix~\ref{sec:assumptiondisc}).
\begin{ass}\label{ass:kernelconc}
$m = \tilde{\Omega}(\epsilon^{-2})$ suffices to ensure that $\norm{K_0 - K^{\infty}}_{L^2(X \times X, \rho \otimes \rho)}^2 \leq \epsilon$ holds with high probability $1 - \delta(m)$ over the initialization $\theta_0$ where $\delta(m) = o(1)$.
\end{ass}
\noindent
Under this assumption, by setting $K = K^\infty$ in Theorem \ref{thm:main} we get the following corollary.
\begin{cor}\label{cor:maincor}
Let $\delta(m)$ be defined as in Assumption~\ref{ass:kernelconc} which we assume to hold.  Let $T \geq 1$ and $\epsilon > 0$.  For $k \in \NN$ let $P_k : L^2(X, \rho) \rightarrow L^2(X, \rho)$ denote the orthogonal projection onto the span of the top $k$ eigenfunctions of the operator $T_{K^\infty}$ defined in Equation~\eqref{eq:tkinftydef}.  Let $\sigma_k > 0$ denote the $k$-th eigenvalue of $T_{K^\infty}$.  Then $m = \tilde{\Omega}(T^4 / \epsilon^2)$ and $n = \tilde{\Omega}(T^2/\epsilon^2)$ suffices to ensure with probability at least $1 - O(mn)\exp(-\Omega(\log^2(m)) - \delta(m)$ that for all $t \leq T$ and $k \in \NN$
\[\norm{P_k(r_t - \exp(-T_{K^\infty}t)r_0)}_{L^2(X, \rho)}^2 
\leq \brackets{\frac{1 - \exp(-\sigma_k t)}{\sigma_k}}^2 \cdot \epsilon \]
and
\[\norm{r_t - \exp(-T_{K^\infty}t)r_0}_{L^2(X, \rho)}^2 \leq t^2 \cdot \epsilon.\]
\end{cor}
\noindent
Informally Corollary \ref{cor:maincor} states that up to the stopping time $T$, we have that $r_t \approx \exp(-T_{K^\infty} t) r_0$.  As discussed before, the term $\exp(-T_{K^\infty} t) r_0$ projected along the $i$-th eigenfunction of $K^\infty$ decays linearly, $\langle \exp(-T_{K^\infty} t) r_0, \phi_i \rangle_\rho = \exp(-\sigma_i t) \langle r_0, \phi_i \rangle_\rho$. Given that $K^\infty$ tends to have a highly skewed spectrum (see, e.g.\ Figure~\ref{fig:ntk_spec}), the effect the magnitude of $\sigma_i$ has on the convergence rate is particularly relevant.  Furthermore the bound on the projected difference $\norm{P_k(r_t - \exp(-T_{K^\infty}t)r_0)}_{L^2(X, \rho)}$ is smaller whenever $\sigma_k$ is large due to the dependence of the bound on the inverse eigenvalue $\sigma_k^{-1}$.  Thus we have that the bias along the top eigenfunctions is particularly pronounced.  Hence we make the following important observation.
\begin{obs}\label{obs:spectralbias}
At the beginning of training the network learns projections along eigenfunctions of the Neural Tangent Kernel integral operator $T_{K^\infty}$ at rates corresponding to their eigenvalues.  This is particularly true for the eigenfunctions with large eigenvalues.
\end{obs}
\paragraph{Scaling with respect to width and number of training data samples}
Now let us interpret how the width $m$ and number of training samples $n$ in the theorem scale.  We note that as long as $n \leq m^{\alpha}$ for some $\alpha > 0$ the failure probability $O(mn) \exp(-\Omega(\log^2(m)))$ goes to zero as $m \rightarrow \infty$.  Thus once $m$ and $n$ are sufficiently large relative to the stopping time $T$ and precision $\epsilon$, they can both tend to infinity at just about any rate to achieve a high probability bound.  We also observe that $m$ and $n$ both have the same scaling with respect to $\epsilon$, namely $m, n = \tilde{\Omega}(\epsilon^{-2})$.  Thus for a fixed stopping time $T$ we can send $m$ and $n$ to infinity at the same rate $m \sim n$ to send the error $\epsilon \rightarrow 0$.  This is significant as typical NTK analysis requires $m = \Omega(poly(n))$.  We reach following important conclusion.  
\begin{obs}\label{obs:linearover}
\textit{The network will inherit the bias of the kernel at the beginning of training even when the width $m$ only grows linearly with the number of samples $n$.}
\end{obs}

\paragraph{Scaling with respect to stopping time}  
We will now address the scaling with respect to the stopping time $T$.  The relevant question is how quickly the terms $P_k \exp(-T_{K^\infty}t)r_0 $ and $\exp(-T_{K^\infty}t)r_0$ converge to zero.  We observe that
\[ \norm{P_k \exp(-T_{K^\infty}t) r_0}_{L^2(X, \rho)} \leq \exp(-\sigma_k t) \norm{r_0}_{L^2(X, \rho)} \leq \exp(-\sigma_k t) \norm{f^*}_{L^\infty(X, \rho)} , \] 
where we have used the antisymmetric initialization $r_0 = f(\bullet; \theta_0) - f^* = 0 - f^* = -f^*$ 
and the basic inequality $\norm{\bullet}_{L^2(X, \rho)} \leq \norm{\bullet}_{L^\infty(X, \rho)}$.  Based on this we have that $t \geq \log(\norm{f^*}_{L^\infty(X, \rho)}/\epsilon) / \sigma_k$ suffices to ensure $\norm{P_k \exp(-T_{K^\infty}t) r_0}_{L^2(X, \rho)} \leq \epsilon$.  Using this fact we get the following corollary.
\begin{cor}\label{cor:testerrbd}
Let $\delta(m)$ be defined as in Assumption~\ref{ass:kernelconc} which is assumed to hold.  Let $T = \tilde{\Omega}(1/\sigma_k)$ and $\epsilon > 0$.  For $k \in \NN$ let $P_k : L^2(X, \rho) \rightarrow L^2(X, \rho)$ denote the orthogonal projection onto the span of the top $k$ eigenfunctions of the operator $T_{K^\infty}$ defined in Equation~\eqref{eq:tkinftydef}.  Let $\sigma_k > 0$ denote the $k$-th eigenvalue of $T_{K^\infty}$.  Then $m = \tilde{\Omega}(\sigma_k{^{-8}} / \epsilon^2)$ and $n = \tilde{\Omega}(\sigma_k^{-6}/\epsilon^2)$ suffices to ensure that with probability at least $1 - O(mn)\exp(-\Omega(\log^2(m)) - \delta(m)$
\[ \norm{P_k r_T}_{L^2(X, \rho)}^2 \leq \epsilon \]
and in particular
\[ \frac{1}{2} \norm{r_T}_{L^2(X, \rho)}^2 \leq \tilde{O}(\epsilon) + \norm{(I - P_k)r_0}_{L^2(X, \rho)}^2. \]
\end{cor}
\noindent
The interpretation of the Corollary~\ref{cor:testerrbd} is that the stopping time $T = \tilde{\Omega}(1/\sigma_k)$ is long enough to ensure that the network has learned the top $k$ eigenfunctions to $\epsilon$ accuracy provided that $m = \tilde{\Omega}(\sigma_k^{-8} \epsilon^{-2})$ and $n = \tilde{\Omega}(\sigma_k^{-6} \epsilon^{-2})$.  We note that the second conclusion of Corollary~\ref{cor:testerrbd} is a bound on the test error $\frac{1}{2} \norm{r_t}_{L^2(X, \rho)}^2$.  From the antisymmetric initialization $r_0 = -f^*$ so that $\norm{(I - P_k) r_0}_{L^2(X, \rho)}^2 = \norm{(I - P_k) f^*}_{L^2(X, \rho)}^2$.  For a general target $f^*$, this quantity can decay arbitrary slowly with respect to $k$.  Our goal with Theorem~\ref{thm:main} was not to get a learning guarantee, but to describe how the bias of the kernel $K^\infty$ is inherited by the finite-width network at the beginning of training even for general target functions.  Nevertheless we will briefly sketch how it is possible to get a learning guarantee from Corollary~\ref{cor:maincor} when $f^*$ is in the RKHS of $K^\infty$.  
In this case one can show that $\norm{\exp(-T_{K^\infty} t) r_0}_{L^2(X, \rho)}^2 = O\parens{\frac{\norm{f^*}_{\mathcal{H}}^2}{t}}$ where $\norm{\bullet}_{\mathcal{H}}$ is the RKHS norm.  Then treating $\norm{f^*}_{\mathcal{H}}$ as a constant one can choose the stopping time $T \sim \epsilon^{-1}$ to bring the test error to $\epsilon$ provided that $m, n = \tilde{\Omega}(poly(\epsilon^{-1}))$.  More generally \citet{yarotskyasymptotics} derive sufficient conditions for the power law $\norm{\exp(-T_{K^\infty} t) r_0}_{L^2(X, \rho)}^2 \sim C t^{-\xi}$ to hold.  Using a similar argument in this case one can choose the stopping time $T \sim \epsilon^{-1/\xi}$ and get a learning guarantee for $m,n = \tilde{\Omega}(poly(\epsilon^{-1}))$.

\subsection{Technical Comparison to Prior Work}
\citet{Lee2019WideNN,aroraexact} compared the network $f(x; \theta)$ to its linearization $f_{lin}(x; \theta) := \langle \nabla_\theta f(x; \theta_0), \theta - \theta_0 \rangle + f(x;\theta_0)$ in the regime where $m = \Omega(poly(n))$.  When $m = \Omega(poly(n))$ one can show the loss converges to zero and the parameter changes $\norm{\theta_t - \theta_0}_2$ are bounded.  By contrast we avoid the condition $m = \Omega(poly(n))$ by employing a stopping time.  \citet{arora2019finegrained, caounderstanding, basri2020frequency} proved statements similar to Theorem~\ref{thm:main} and Corollary~\ref{cor:maincor} that roughly correspond to replacing $T_{K^\infty}$ with its Gram matrix induced by the training data $(G^\infty)_{i,j} = K^\infty(x_i, x_j)$ and replacing $\rho$ with the empirical measure $\hat{\rho} = \frac{1}{n} \sum_{i = 1}^n \delta_{x_i}$.  \citet{arora2019finegrained, basri2020frequency} operate in the regime where $m = \Omega(poly(n))$ and as a benefit do not need to employ a stopping time.  \citet{caounderstanding} instead of requiring $m = \Omega(poly(n))$ requires that the width $m$ satisfies at least $m = \Omega(\max\{\sigma_k^{-14}, \epsilon^{-6}\})$ where $\sigma_k$ is the cutoff eigenvalue.  The most similar work is \citet{bowman2022implicit}, which demonstrated a version of Corollary~\ref{cor:maincor} for a shallow feedforward network that is underparameterized.  If $p$ is the total number of parameters, they require $m = \tilde{\Omega}(\epsilon^{-1} T^2)$ and $n = \tilde{\Omega}(\epsilon^{-1} p T^2)$.  This requires the network to be greatly underparameterized $n \gg p$.  Our result was able to remove the dependence of $n$ on $p$ and demonstrate the result for general deep architectures at the expense of slightly worse scaling with respect to $T$ and $\epsilon$.
\section{Proof Sketch}
For simplicity we will go through the case where $K = K^\infty$.  At a high level the proof revolves around bounding the difference between the operators $T_{K^\infty}$ and $T_n^t$ defined in Equations (\ref{eq:tkinftydef}) and~(\ref{eq:tntdefinition}). 
\paragraph{Bounding Operator Deviations}
\citet{bowman2022implicit} demonstrated
\[ r_t = \exp(-T_{K^\infty} t) r_0 + \int_0^t \exp(-T_{K^\infty}(t - s))(T_{K^\infty} - T_n^s)r_s ds. \]
This exhibits the residual $r_t$ as a sum of $\exp(-T_{K^\infty} t) r_0$ and a correction term.  The proof of Theorem \ref{thm:main} revolves around bounding the correction term which involves bounding
\[\norm{(T_{K^\infty} - T_n^s) r_s}_{L^2(X, \rho)} \leq \norm{(T_{K^\infty} - T_n)r_s}_{L^2(X, \rho)} + \norm{(T_n - T_n^s) r_s}_{L^2(X, \rho)}. \]
At a high level $\norm{(T_n - T_n^s) r_s}_{L^2(X, \rho)}$ will be small whenever the kernel deviations $K_0 - K_s$ are small.  On the other hand by metric entropy based arguments we have that $\norm{(T_{K^\infty} - T_n)r_s}_{L^2(X, \rho)}$ will be small whenever $n$ is large enough relative to the complexity of the residual functions $r_s$.
\paragraph{Comparison with Linearization}
  Let $H(x; \theta) := \nabla_\theta^2 f(x; \theta)$ denote the Hessian of our network with respect to the parameters $\theta$ for a fixed input $x$.  It turns out that if $\norm{H(x, \theta)}_{op}$ was uniformly small over $x$ and $\theta$ then the kernel deviations $K_0 - K_s$ would be bounded and the complexity of our model $f(x; \theta)$ would be controlled by the complexity of the linearized model $f_{lin}(x; \theta) := \langle \nabla_\theta f(x; \theta_0), \theta - \theta_0 \rangle$.  The caveat to this approach is we do not in fact have a way to bound the Hessian $H(x, \theta)$ uniformly.  However \citet{liuonlinearity} demonstrated that for \textit{fixed} $x$ and $R > 0$ we have with high probability over the initialization $\theta_0$
 \begin{equation}\label{eq:liuhess}
 \sup_{\theta \in \overline{B}(\theta_0, R)} \norm{H(x, \theta)}_{op} = \tilde{O}\parens{\frac{R}{\sqrt{m}} poly(R/\sqrt{m})}.     
 \end{equation}
 \noindent
 Using a priori parameter norm deviation bounds we have that $\norm{\theta_t - \theta_0}_2 = O(\sqrt{t})$ and thus we can set $R = O(\sqrt{T})$.  The difficulty then arises to get bounds that only depend on the Hessian $H(x; \theta)$ evaluated only on finitely many inputs $x$.  We overcome this difficulty by showing for fixed $\theta_0$ one has high probability bounds over the sampling of the training data $x_1, \ldots, x_n$ that only require the Hessian evaluated on a finite point set.  This requires some elaborate calculations involving Rademacher complexity.  We then use the Fubini-Tonelli theorem and the Hessian bound \eqref{eq:liuhess} to get a bound over the simultaneous sampling of $\theta_0$ and $x_1, \ldots, x_n$.
\paragraph{Covering Number of the Linearized Model}
The complexity of the residual functions $r_s$ up to the stopping time $T$ can be controlled by bounding the complexity of the function class $\mathcal{C} = \{f_{lin}(x; \theta) : \theta \in \overline{B}(\theta_0, R)\}$.  In Appendix~\ref{sec:covnum} we show that the $L^2(X, \rho)$ metric entropy of the linearized model $\mathcal{C} = \{f_{lin}(x; \theta) : \theta \in \overline{B}(\theta_0, R)\}$ is determined by the spectrum of the Fisher Information Matrix
\begin{equation}\label{eq:fisherdefinition}
F := \int_X \nabla_\theta f(x; \theta_0) \nabla_\theta f(x; \theta_0)^T d\rho(x).
\end{equation}
Let $\lambda_1^{1/2} \geq \lambda_2^{1/2} \geq \cdots \geq 0$ denote the eigenvalues of $F^{1/2}$.  We define the effective rank of $F^{1/2}$ at scale $\epsilon$ as
\[ \Tilde{p}(F^{1/2}, \epsilon) = |\{i : \lambda_i^{1/2} > \epsilon\}|. \] This measures the number of dimensions within the unit ball whose image under $F^{1/2}$ can be larger than $\epsilon$ in Euclidean norm.  In Appendix~\ref{sec:covnum} we demonstrate that the $\epsilon$ covering number of $\mathcal{C}$ in $L^2(X, \rho)$, denoted $\mathcal{N}(\mathcal{C}, \norm{\bullet}_{L^2(X, \rho)}, \epsilon)$, has the bound
\[ \log \mathcal{N}(\mathcal{C}, \norm{\bullet}_{L^2(X, \rho)}, \epsilon) = \tilde{O}(\tilde{p}(F^{1/2}, 0.75 \epsilon / R)).  \]
It turns out that for $\norm{(T_{K^\infty} - T_n)r_s}_{L^2(X, \rho)}$ to be on the order of $\epsilon$ we merely need $n$ to be large relative to $\tilde{p}(F^{1/2}, 0.75 \epsilon / R)$.  By contrast \citet{bowman2022implicit} required that the network was underparameterized so that $n$ was large relative to the total number of parameters $p$.  Since $\tilde{p} \ll p$, this is what lets us relax the sample complexity dramatically.  In fact for fixed $R$ and $\epsilon$ we have that $\tilde{p} = \tilde{O}(1)$ with high probability as the width grows to infinity whereas $p \rightarrow \infty$.  Interestingly, the quantity $\tilde{p}$ for the loss Hessian at convergence was used recently to derive analytical PAC-Bayes bounds \citep{Yang2021DoesTD}.  Note for the squared loss the (empirical) FIM\footnote{Note that we define $F$ as an expectation over the true input distribution $\rho$.  To approximate the Hessian of the empirical loss one must replace $\rho$ with the empirical measure $\hat{\rho}$.} can be taken as an approximation to the Hessian, and at a minimizer this approximation becomes exact.  Thus these two notions are closely related.

\section{Conclusion and Future Directions}
\label{sec:conclusion}
We provided quantitative bounds measuring the $L^2$ difference in function space between a finite-width network trained on finitely many samples and the corresponding kernel method with infinite width and infinite data.  As a consequence, the network will inherit the bias of the kernel at the beginning of training even when the width scales linearly with the number of samples.  This bias is not only over the training data but over the entire input space.  The key property that allows this is the low-effective-rank property of the Fisher Information Matrix (FIM) at initialization which controls the capacity of the model at the beginning of training.  An interesting avenue for future work is to investigate if flat minima manifesting a FIM of low effective rank at the end of training can be related to the behavior of the network on out-of-sample data after training.

\paragraph{Limitations}
Our framework can only characterize the network's bias up to a stopping time.  There is compelling evidence that the kernel adapts to the target function later in training \citep{kernelalignment,atanasov2022neural}, and this falls outside our framework.  Accounting for adaptations in the kernel is an important problem that is still being addressed by the theoretical community.

\paragraph{Broader impacts} 
We do not foresee any negative societal impacts of characterizing the spectral bias of neural networks.  To the contrary we believe that cataloging the properties that networks are biased towards in a variety of regimes will be essential to developing fair and interpretable artificial intelligence over the long-term.

\ifdefined\deanonymize
\begin{ack}
This project has received funding from UCLA FCDA and from the European Research Council (ERC) under the European Union’s Horizon 2020 research and innovation programme (grant agreement n\textsuperscript{o}~757983).  The authors would like to thank Yonatan Dukler for sharing code to compute the NTK Gram matrix in PyTorch.
\end{ack}
\fi


\bibliography{deep_mse}
\bibliographystyle{deep_mse}

\appendix

\newpage
\section*{Appendix}

The appendix is organized as follows. 
\begin{itemize}[leftmargin=*]
\item In Appendix~\ref{sec:covnum} we bound the $L^2(X, \rho)$ metric entropy of the linearized model.  This is necessary to bound the operator deviation $T_K - T_n$.
\item In Appendix~\ref{sec:hessianbd} we bound the Hessian of the network and introduce some technical lemmas.  This is necessary in order to relate the network to the linearized model.
\item In Appendix~\ref{sec:conv_operators} we bound the quantity $\norm{(T_K - T_n^t) r_t}_{L^2(X, \rho)}$.  This section contains the bulk of the proof for the main result Theorem~\ref{thm:main}.  
\item In Appendix~\ref{sec:main_result_proof} we put the aforementioned results together to prove Theorem~\ref{thm:main}.
\item In Appendix~\ref{sec:assumptiondisc} we explain the merit of Assumption~\ref{ass:kernelconc}.
\item In Appendix~\ref{sec:experimental} we describe the details of our experiments with a link to the relevant code.
\end{itemize}

\section{Covering Number for the Linearized Model}\label{sec:covnum}
Our approach to generalization will be based on metric entropy \citep[see, e.g.,][]{wainwright_2019}, a fundamental tool in learning theory. 
We recall some basic definitions. 
\begin{defi}
Let $V$ be a vector space with seminorm $\norm{\bullet}$.  For a subset $A \subset V$ we say that $B$ is a proper $\epsilon$-covering of $A$ if $B \subset A$ and for all $a \in A$ there exists $b \in B$ such that $\norm{a - b} \leq \epsilon$.
\end{defi}
Since we will concern ourselves solely with proper coverings we may remove the adjective ``proper'' when discussing $\epsilon$-coverings.  A closely related notion is the $\epsilon$-covering number.
\begin{defi}
Let $V$ be a vector space with seminorm $\norm{\bullet}$ and let $A \subset V$.  For $\epsilon > 0$ we define the proper $\epsilon$-covering number of $A$, denoted $\mathcal{N}(A, \norm{\bullet}, \epsilon)$, by
\[ \mathcal{N}(A, \norm{\bullet}, \epsilon) = \min_{\text{$N$ : $N$ is proper $\epsilon$-covering of $A$}} |N|. \]
\end{defi}
It is also useful to define the covering number of a set $K$ with respect to another set $L$.
\begin{defi}
Let $K$ and $L$ be two subsets of a vector space $V$.  We define $\mathcal{N}(K, L)$ as the smallest $n \in \NN$ such that there exists $v^{(1)}, \ldots, v^{(n)} \in K$ satisfying
\[ K \subset \bigcup_{i = 1}^n (v^{(i)} + L). \]
\end{defi}
Now consider a model $f_{lin}(x; \theta)$ that is potentially nonlinear in $x$ but affine in $\theta$.  The motivating example is the following NTK model
\[ f_{lin}(x;\theta) = f(x;\theta_0) + \langle \nabla_\theta f(x;\theta_0), \theta - \theta_0 \rangle. \]
We will be interested in deriving covering numbers for such classes of functions.  Since translation by a fixed function does not change the covering number we will for convenience assume the model is linear in $\theta$.  Thus we will consider models of the form
\[ f_{lin}(x;\theta) = \langle g(x), \theta \rangle. \]
The function $g$ can be nonlinear and thus $x \mapsto f_{lin}(x;\theta)$ is typically nonlinear.  For the NTK model we have $g(x) = \nabla_\theta f(x; \theta_0)$.  Let $X$ be our input space and let $\nu$ be some measure on $X$.  We consider $L^2(X, \nu)$ where
\[ \norm{h}_{L^2(X, \nu)}^2 = \int_X |h(x)|^2 d\nu(x). \]
Throughout we will assume that $\norm{g}_2 \in L^2(X, \nu)$ i.e. $\int_X \norm{g}_2^2 d\nu < \infty$.  We will be interested in deriving covering numbers for classes of functions
\[ \mathcal{C}_A := \{ f_{lin}(x; \theta) : \theta \in A\} \]
where $A \subset \Theta$ is some subset of parameter space $\Theta$.  For now we will assume that $\Theta = \RR^p$.  We observe that
\begin{gather*}
\norm{f_{lin}(\bullet; \theta_1) - f_{lin}(\bullet;\theta_2)}_{L^2(X, \nu)}^2  = \int_X |\langle g(x), \theta_1 - \theta_2 \rangle|^2 d\nu(x) \\
= \int_X (\theta_1 - \theta_2)^T g(x) g(x)^T (\theta_1 - \theta_2) d\nu(x) = (\theta_1 - \theta_2)^T \brackets{\int_X g(x) g(x)^T d\nu(x)} (\theta_1 - \theta_2).
\end{gather*}
Thus of primary importance is the symmetric positive semidefinite matrix $M := \int_X g(x) g(x)^T d\nu$.  When $\nu$ is a probability measure and $f_{lin}(x;\theta)$ is the NTK model we have that 
\[ M = \EE_{x \sim \nu}\brackets{\nabla_\theta f(x; \theta_0) \nabla_\theta f(x;\theta_0)^T} \]
is the (uncentered) gradient covariance matrix, which can be interpreted as the Fisher Information Matrix (FIM) for the squared loss.  The two most interesting cases are when $\nu$ is the true input distribution or $\nu = \frac{1}{n} \sum_{i = 1}^n \delta_{x_i}$ is the empirical distribution arising from the training samples.  In the former case $M$ is the true (uncentered) gradient covariance matrix and in the latter case $M$ is the (uncentered) empirical covariance.  For neural networks the FIM tends to have a very skewed spectrum (is approximately low rank), and thus the relations between the spectrum of $M$ and the covering number will be particularly relevant.  We will define the seminorm $\norm{\bullet}_M$ as
\[ \norm{v}_M := \sqrt{v^T M v}. \]
The following lemma relates the covering number $\mathcal{N}(\mathcal{C}_A, \norm{\bullet}_{L^2(X, \nu)}, \epsilon)$ to $\mathcal{N}(A, \norm{\bullet}_M, \epsilon)$.
\begin{lem}\label{lem:l2coveriseuclidcover}
Let $N \subset A \subset \RR^p$.  Then $N$ is a proper $\epsilon$-covering of $A$ with respect to the seminorm $\norm{\bullet}_M$ if and only if $\mathcal{C}_N$ is a proper $\epsilon$-covering of $\mathcal{C}_A$ with respect to the $L^2(X, \nu)$ norm.
\end{lem}
\begin{proof}
As we argued before we have that
\begin{gather*}
\norm{f_{lin}(\bullet; \theta_1) - f_{lin}(\bullet; \theta_2)}_{L^2(X, \nu)}^2  = (\theta_1 - \theta_2)^T \brackets{\int_X g(x) g(x)^T d\nu(x)} (\theta_1 - \theta_2) \\
= (\theta_1 - \theta_2)^T M (\theta_1 - \theta_2) = \norm{\theta_1 - \theta_2}_M^2. 
\end{gather*}
For each function in $h \in \mathcal{C}_A$ pick a representative parameter $\hat{\theta}(h) \in A$ so that $h = f_{lin}(\bullet; \hat{\theta}(h))$ (if $M$ is strictly positive definite $\hat{\theta}(h)$ is unique).  We can choose the mapping $h \mapsto \hat{\theta}(h)$ so that the image of $\mathcal{C}_N$ under this mapping is $N$.  Suppose $N$ is an $\epsilon$-covering for A with respect to $\norm{\bullet}_M$.  Then for each $\theta \in A$ we can choose $\theta'$ such that $\|\theta - \theta'\|_M \leq \epsilon$.  Well then for any $h \in \mathcal{C}_A$ we can consider $\hat{\theta}(h)$ and choose $\theta' \in N$ such that $\epsilon \geq \big\|\hat{\theta}(h) - \theta'\big\|_M = \big\| f_{lin}(\bullet;\hat{\theta}(h)) - f_{lin}(\bullet;\theta')\big\|_{L^2(X, \nu)}$.  It follows that $\mathcal{C}_N$ is an $\epsilon$-covering of $\mathcal{C}_A$.  Conversely suppose now that $\mathcal{C}_N$ is an $\epsilon$-covering of $\mathcal{C}_A$ with respect to $\norm{\bullet}_{L^2(X, \nu)}$.  Well then for any $\theta \in A$ we can consider $f_{lin}(x;\theta)$ and take $h \in \mathcal{C}_N$ such that $\norm{f_{lin}(\bullet;\theta) - h(\bullet)}_{L^2(X, \nu)} \leq \epsilon$.  However since $h(\bullet) = f_{lin}(\bullet;\hat{\theta}(h))$ we have that $\epsilon \geq \big\|f_{lin}(\bullet;\theta) - f_{lin}(\bullet; \hat{\theta}(h))\big\|_2 = \big\|\theta - \hat{\theta}(h)\big\|_M$.  Thus $\hat{\theta}(\mathcal{C}_N) = N$ is an $\epsilon$-covering for $A$.
\end{proof}
Thus covering the space $\mathcal{C}_A$ in $L^2(X, \nu)$ reduces to covering a subset of Euclidean space under the seminorm $\norm{\bullet}_M$.  By a change of coordinates we will assume without loss of generality that $M$ is diagonal.  Let $M^{1/2}$ be the square root of $M$ and let $\sigma_1 \geq \cdots \geq \sigma_p \geq 0$ be the eigenvalues of $M^{1/2}$.  We note that
\[ \braces{v \in \RR^p : \norm{v}_M \leq 1} = \braces{ v \in \RR^p : \sum_{i = 1}^p \sigma_i^2 v_i^2 \leq 1}. \]
Thus the unit ball in $\RR^p$ determined by $\norm{\bullet}_M$ is the ellipsoid with half-axis lengths $\sigma_i^{-1}$ (if $\sigma_i = 0$ we consider the ellipsoid as being infinite along that dimension).  For a general vector $a \in \RR^p$ with nonnegative entries we define the ellipse
\[ E_a := \braces{v \in \RR^p : \sum_{i = 1}^p \frac{v_i^2}{a_i^2} \leq 1} \]
where in the sum if $a_i = 0$ we interpret $\frac{v_i^2}{a_i^2}$ as $0$ if $v_i = 0$ and infinity otherwise.  $E_a$ is the ellipse with half-axis lengths $a_1, a_2, \ldots, a_n$.  We will also let $B_r^k \subset \RR^k$ denote the closed Euclidean ball in dimension $k$ of radius $r$, specifically
\[ B_r^k := \{v \in \RR^k : \sum_{i = 1}^k v_i^2 \leq r \}. \]
Our main study will be bounding $\mathcal{N}(A, \norm{\bullet}_M, \epsilon)$ when $A = \{ \theta \in \RR^p : \norm{\theta}_2 \leq R\} = B_R^p$.  This amounts to covering a Euclidean ball with ellipsoids determined by $\norm{\bullet}_M$.  Fortunately, there are well established results for coverings involving ellipsoids.  Let $\sigma = (\sigma_1, \ldots, \sigma_p)^T$ denote the spectrum of $M^{1/2}$ and let $M^{-1/2}$ denote the pseudo-inverse of $M^{1/2}$.  Let $L$ denote the closed unit ball in $\RR^p$ under the seminorm $\norm{\bullet}_M$.  In geometric terms $\mathcal{N}(B_R^p, \norm{\bullet}_M, \epsilon) = \mathcal{N}(B_R^p, \epsilon L)$.  We claim that up to an application of $M^{1/2}$ or $M^{-1/2}$, covering $B_R^p$ with translates of $\epsilon L$ is equivalent to covering $E_{\frac{R}{\epsilon} \sigma}$ with translates of $B_1^p$.  This is formalized in the following lemma.
\begin{lem}\label{lem:ballcovertoellcover}
Let $M \in \RR^{p \times p}$ be a symmetric positive semidefinite matrix and let $\sigma = (\sigma_1, \ldots, \sigma_p)^T \in \RR^p$ denote the eigenvalues of $M^{1/2}$.  Then $\mathcal{N}(B_R^p, \norm{\bullet}_M, \epsilon) = \mathcal{N}(E_{\frac{R}{\epsilon} \sigma}, B_1^p)$.
\end{lem}
\begin{proof}
By a change of basis we can assume without loss of generality that $M$ is diagonal. Let $L$ denote the closed unit ball of $\RR^p$ under $\norm{\bullet}_M$.  We note that in geometric terms $\mathcal{N}(B_R^p, \norm{\bullet}_M, \epsilon) = \mathcal{N}(B_R^p, \epsilon L)$.  Since we can dilate by $1/\epsilon$ we can replace $R$ with $R / \epsilon$ and $\epsilon$ with $1$.  Thus for convenience we will assume for now that $\epsilon = 1$.  We note that if $v^{(1)}, \ldots, v^{(n)}$ form an $L$ covering of $B_R^p$ as in
\[ B_R^p \subset \bigcup_{i = 1}^n (v^{(i)} + L), \]
then
\[ E_{R \sigma} = M^{1/2}(B_R^p) \subset \bigcup_{i = 1}^n (M^{1/2}v^{(i)} + M^{1/2}(L)) \subset \bigcup_{i = 1}^n (M^{1/2} v^{(i)} + B_1^p). \]
Thus $M^{1/2}v^{(1)}, \ldots, M^{1/2}v^{(n)}$ forms a $B_1^p$ covering of $E_{R \sigma}$.  Conversely suppose $v^{(1)}, \ldots, v^{(n)}$ satisfy
\[ E_{R \sigma} \subset \bigcup_{i = 1}^n (v^{(i)} + B_1^p) \]
and let $P$ be the projection onto $\text{span}\{e_i : \sigma_i \neq 0 \}$ where $e_i$ denotes the $i$th standard basis vector.  Then
\[ P(B_{R}^p) = M^{-1/2}(E_{R \sigma}) \subset \bigcup_{i = 1}^n (M^{-1/2}v^{(i)} + M^{-1/2}(B_1^p)) = \bigcup_{i = 1}^n (M^{-1/2}v^{(i)} + P(L)). \]
However $L$ is infinitely long along the dimensions outside $im(P)$, and thus
\[B_{R}^p \subset \bigcup_{i = 1}^n (M^{-1/2}v^{(i)} + L). \]
Thus $M^{-1/2} v^{(1)}, \ldots, M^{-1/2} v^{(n)}$ form an $L$ covering of $B_R^p$.  We conclude that $\mathcal{N}(B_R^p, L) = \mathcal{N}(E_{R \sigma}, B_1^p)$.  Thus for general $\epsilon > 0$ we have that $\mathcal{N}(B_R^p, \norm{\bullet}_M, \epsilon) = \mathcal{N}(B_R^p, \epsilon L) = \mathcal{N}(B_{R / \epsilon}^p, L) = \mathcal{N}(E_{\frac{R}{\epsilon} \sigma}, B_1^p)$.
\end{proof}
\par
We will let $vol(\bullet)$ denote volume in the standard Lebesgue sense.  If $a \in \RR^p$ is a vector with positive entries we recall that the volume of an ellipsoid $E_a$ is given by the formula
\[ vol(E_a) = vol(B_1^p) \prod_{i = 1}^p a_i . \]
When most of the $a_i$ are very small we have that $E_a$ is very thin and has small volume and thus we expect the covering number to be small.  Coverings for ellipsoids are well established with roots in geometric functional analysis.  The following lemma is phrased the same as Theorems 1 and 2 in \citet{dumer_ellipsoid}.  The result dates back to classic results in geometric functional analysis.  Specifically a similar result for more general convex bodies is sketched at the end of Chapter 5 in \citet{pisier_1989} which also appeared in \citet[Proposition 1.7]{Gordon1987GeometricAP}.  We don't need the additional generality for our purposes.  We will offer the simplest proof needed for our purposes for completeness and clarity.
\begin{lem}[{\citealt{dumer_ellipsoid, pisier_1989, Gordon1987GeometricAP}}]
\label{lem:ellipsoidcoveringnum} 
Let $a \in \RR^p$ be a vector with nonnegative entries.  Let $J = \{i : a_i > 1\}$, $K = \sum_{i \in J} \log(a_i)$, $\gamma \in (0, 1/2)$, and $\mu_\gamma = |\{i : a_i^2 > (1 - \gamma)^2\}|$.  Then the proper covering number $\mathcal{N}(E_a, B_1^p)$ satisfies
\[ K \leq \log \mathcal{N}(E_a, B_1^p) \leq K + \mu_\gamma \log \parens{\frac{3}{\gamma}}. \]
\end{lem}
\begin{proof}
We first prove the lower bound.  Let $J = \{i : a_i > 1\}$, $m = |J|$, and let $P$ be the orthogonal projection onto $\text{span}\{e_i : i \in J\}$ where $e_i$ denotes the standard basis.  Suppose $v^{(1)}, \ldots, v^{(n)}$ are the centers of a $B_1^p$ covering of $E_a$, specifically 
\[ E_a \subset \bigcup_{i = 1}^n (v^{(i)} + B_1^p) . \]
Well then 
\[ P(E_a) \subset \bigcup_{i = 1}^n P(v^{(i)} + B_1^p) = \bigcup_{i = 1}^n (Pv^{(i)} + B_1^m).\]
Well then by the standard volume estimate we get that
\[ n \cdot vol(B_1^m) \geq vol \parens{\bigcup_{i = 1}^n (Pv^{(i)} + B_1^m)} \geq vol(P(E_a)) \]
and thus
\[ n \geq \frac{vol(P(E_a))}{vol(B_1^m)} = \prod_{i \in J} a_i. \]
\par
Now we prove the upper bound.  Let $\gamma \in (0, 1/2)$ and let $J_\gamma = \{i : a_i^2 > (1 - \gamma)^2 \}$, $\mu_\gamma = |J_\gamma|$, and let $P$ be the orthogonal projection onto $\text{span}\{e_i : i \in J_\gamma \}$.  We first notice that if $v \in E_a$ we have that $\norm{(I - P) v}_2 \leq 1 - \gamma$, indeed because for $v \in E_a$
\[ \sum_{i \notin J_\gamma} \frac{v_i^2}{(1 - \gamma)^2} \leq \sum_{i \notin J_\gamma} \frac{v_i^2}{a_i^2} \leq 1. \]
Thus if $v^{(1)}, \ldots, v^{(n)}$ are the centers of a proper $B_\gamma^{\mu_\gamma}$ covering of $P(E_a)$ then by the triangle inequality they also induce a proper $B_1^p$ covering of $E_a$.  Thus let $v^{(1)}, \ldots, v^{(n)}$ be a maximal subset of $P(E_a)$ such that for $i \neq j$ $\norm{v^{(i)} - v^{(j)}}_2 > \gamma$.  By maximality $v^{(1)}, \ldots, v^{(n)}$ form a $B_\gamma^{\mu_\gamma}$ covering of $P(E_a)$.  Well then the balls $v^{(i)} + B_{\gamma / 2}^{\mu_\gamma}$ are all disjoint and contained in $P(E_a) + B_{\gamma / 2}^{\mu_\gamma}$.  Thus by the volume estimates
\[ n \cdot vol(B_{\gamma/2}^{\mu_\gamma}) = vol\parens{\bigcup_{i = 1}^n (v^{(i)} + B_{\gamma / 2}^{\mu_\gamma})} \leq vol \parens{P(E_a) + B_{\gamma / 2}^{\mu_\gamma}}. \]
Thus 
\[ n \leq \frac{vol\parens{P(E_a) + B_{\gamma / 2}^{\mu_\gamma}}}{vol(B_{\gamma/2}^{\mu_\gamma})}. \]
Note that $B_{1 - \gamma}^{\mu_\gamma} \subset P(E_a)$ and thus $B_{\gamma /2}^{\mu_\gamma} \subset \frac{\gamma}{2(1-\gamma)} P(E_a)$.  Now let $\norm{\bullet}_{P(E_a)}$ be the norm on $\RR^{\mu_\gamma}$ such that $P(E_a)$ is the unit ball.  Then note for $v,w$ such that $v \in P(E_a)$ and $w \in B_{\gamma/2}^{\mu_\gamma}$ we have that
\[ \norm{v + w}_{P(E_a)} \leq \norm{v}_{P(E_a)} + \norm{w}_{P(E_a)} \leq 1 + \frac{\gamma}{2(1-\gamma)}. \]
We conclude that $P(E_a) + B_{\gamma/2}^{\mu_\gamma} \subset \parens{1 + \frac{\gamma}{2(1-\gamma)}} P(E_a)$.  Therefore
\[ n \leq \frac{vol \parens{P(E_a) + B_{\gamma / 2}^{\mu_\gamma}}}{vol(B_{\gamma/2}^{\mu_\gamma})}  \leq \frac{vol \brackets{\parens{1 + \frac{\gamma}{2(1-\gamma)}} P(E_a)}}{vol(B_{\gamma/2}^{\mu_\gamma})}  = \parens{\frac{2}{\gamma} + \frac{1}{1 - \gamma}}^{\mu_\gamma} \prod_{i \in J_\gamma} a_i. \]
Note that since $\gamma < 1/2$ we have that $\frac{1}{1-\gamma} < \frac{1}{\gamma}$.  Therefore $\frac{2}{\gamma} + \frac{1}{1-\gamma} \leq \frac{3}{\gamma}$.  Moreover $\prod_{i \in J_\gamma} a_i \leq \prod_{i \in J} a_i$.  Thus
\[ n \leq \parens{\frac{2}{\gamma} + \frac{1}{1 - \gamma}}^{\mu_\gamma} \prod_{i \in J_\gamma} a_i \leq \parens{\frac{3}{\gamma}}^{\mu_\gamma} \prod_{i \in J} a_i.  \]
After taking logarithms we get the desired result.
\end{proof}
From the Lemmas \ref{lem:ballcovertoellcover} and \ref{lem:ellipsoidcoveringnum} we see that the covering number $\mathcal{N}(B_R^p, \norm{\bullet}_M, \epsilon)$ will depend on how many eigenvalues of $M$ lie above a certain threshold.  Let $A \in \RR^p$ be a symmetric positive semidefinite square matrix with eigenvalues $\lambda_1 \geq \lambda_2 \geq \cdots \geq \lambda_p \geq 0$.  We define the effective rank of $A$ at scale $\epsilon$ as
\[ \Tilde{p}(A, \epsilon) = |\{i : \lambda_i > \epsilon\}|. \] This measures the number of dimensions within $B_1$ whose image under $A$ can be larger than $\epsilon$ in Euclidean norm.  We will also define 
\[ |A|_{>c} = \prod_{i : \lambda_i > c} \lambda_i , \]
which can be thought of the determinant of $A$ after removing some eigenvalues.  We then have our main result.
\begin{theo}\label{thm:kernelcovnum}
Let $g : X \rightarrow \RR^p$ such that $\norm{g}_2 \in L^2(X, \nu)$.  Let $\mathcal{C} = \{x \mapsto \langle g(x), \theta \rangle : \norm{\theta}_2 \leq R\}$, $\gamma \in (0, 1/2)$.  Define $M \in \RR^{p \times p}$ by
\[ M = \int_X g(x) g(x)^T d\nu(x). \]
Then the proper covering number $\mathcal{N}(\mathcal{C}, \norm{\bullet}_{L^2(X, \nu)}, \epsilon)$ satisfies
\[ \log \abs{\frac{R}{\epsilon} M^{1/2}}_{> 1} \leq \log \mathcal{N}(\mathcal{C}, \norm{\bullet}_{L^2(X, \nu)}, \epsilon) \leq \log \abs{\frac{R}{\epsilon} M^{1/2}}_{> 1} + \Tilde{p}\parens{\frac{R}{\epsilon} M^{1/2}, (1 - \gamma)} \log \parens{\frac{3}{\gamma}}. \]
\end{theo}
\begin{proof}
We have by Lemmas \ref{lem:l2coveriseuclidcover} and \ref{lem:ballcovertoellcover}  that $\mathcal{N}(\mathcal{C}, \norm{\bullet}_{L^2(X, \nu)}, \epsilon) = \mathcal{N}(B_R^p, \norm{\bullet}_M, \epsilon) = \mathcal{N}(E_{\frac{R}{\epsilon} \sigma}, B_1^p)$ where $\sigma = (\sigma_1, \ldots, \sigma_p)^T \in \RR^p$ is the vector of eigenvalues of $M^{1/2}$.  Well then by applying Lemma \ref{lem:ellipsoidcoveringnum} with $a = \frac{R}{\epsilon} \sigma$ we have that
\[ \log \abs{\frac{R}{\epsilon} M^{1/2}}_{> 1} \leq \log \mathcal{N}(E_{\frac{R}{\epsilon} \sigma}, B_1) \leq \log \abs{\frac{R}{\epsilon} M^{1/2}}_{> 1} + \Tilde{p}\parens{\frac{R}{\epsilon} M^{1/2}, (1 - \gamma)} \log \parens{\frac{3}{\gamma}}. \]
The desired result thus follows.
\end{proof}
\begin{cor}\label{cor:covnum}
Let $g : X \rightarrow \RR^p$ such that $\norm{g}_2 \in L^2(X, \nu)$.  Let $\mathcal{C} = \{x \mapsto \langle g(x), \theta \rangle : \norm{\theta}_2 \leq R\}$, $\gamma \in (0, 1/2)$.  Define $M \in \RR^{p \times p}$ by
\[ M = \int_X g(x) g(x)^T d\nu(x). \]
Then the proper covering number $\mathcal{N}(\mathcal{C}, \norm{\bullet}_{L^2(X, \nu)}, \epsilon)$ satisfies
\[\log \mathcal{N}(\mathcal{C}, \norm{\bullet}_{L^2(X, \nu)}, \epsilon) = \tilde{O}\parens{\tilde{p}\parens{M^{1/2}, \frac{3\epsilon}{4R}}}. \]
\end{cor}
\begin{proof}
This follows from setting $\gamma = 1/4$ and the fact that
\begin{align*}
\log\abs{\frac{R}{\epsilon} M^{1/2}} &= 
\log \parens{\prod_{\sigma_i > \epsilon / R} \frac{R}{\epsilon} \sigma_i} \\
&\leq \tilde{p}(M^{1/2}, \epsilon / R) \log\parens{\frac{R \sigma_1}{\epsilon}} \\
&\leq \tilde{p}\parens{M^{1/2}, \frac{3 \epsilon}{4 R}} \log\parens{\frac{R \sigma_1}{\epsilon}}.    
\end{align*}
\end{proof}

\section{Bounding the Network Hessian and other Technical Items}
\label{sec:hessianbd}
\subsection{Main Hessian Bound}
For a fixed input $x$ we will let $H(x, \theta) := \nabla_\theta^2 f(x;\theta)$ denote the Hessian of the network with respect to the parameters.  We will use the following result, which follows from the proof of a result by \citet[Theorem 3.3]{liubelkinarxiv},  which we state here explicitly for reference. 
\begin{theo}[Reformulation of {\citealt[Theorem 3.3]{liubelkinarxiv}}]
\label{thm:liuhessbd} 
Let $f(x; \theta)$ be a general neural network of the form specified in Section~\ref{sec:architecture} which can be a fully connected network, CNN, ResNet or a mixture of these types.  Let $m$ be the minimum of the hidden layer widths and assume $\max_l \frac{m_l}{m} = O(1)$.  Given any fixed $R \geq 1$ and $x \in X$ then with probability at least $1 - Cme^{-c\log^2(m)}$
\[ \sup_{\theta \in \overline{B}(\theta_0, R)} \norm{H(x, \theta)}_{op} = \tilde{O}\parens{\frac{R}{\sqrt{m}} \brackets{\max\braces{1, \frac{R}{\sqrt{m}}}}^{O(L)}}. \]
In particular if $\sqrt{m} \geq R$ then
\[ \sup_{\theta \in \overline{B}(\theta_0, R)} \norm{H(x, \theta)}_{op} = \tilde{O}\parens{\frac{R}{\sqrt{m}}}. \]
The constants $c, C > 0$ depend on the architecture but are independent of the width.
\end{theo}

\paragraph{Discussion of the statement of Theorem~\ref{thm:liuhessbd}} 
We note that our statement of Theorem~\ref{thm:liuhessbd} is not exactly the same as the result of \citet[Theorem 3.3]{liubelkinarxiv}.  \citet{liubelkinarxiv} do not explicitly write the failure probability and the dependence of the Hessian bound on $R$ in the statement of the theorem.  In Theorem~\ref{thm:liuhessbd} we write the failure probability and dependence on the radius $R$ according to the proof\footnote{We communicated with the authors to better understand the dependence of the bound on the quantity $R$.  Nevertheless we accept full liability for any misinterpretation of their proof.} provided by the authors \citet{liubelkinarxiv}.  We also add the assumption $\max_{l} \frac{m_l}{m} = O(1)$ to the hypothesis.  This assumption is so that the initial weight matrices satisfy $\frac{1}{\sqrt{m}} \big\|W_0^{(l)}\big\|_{op} = O(1)$ with high probability (see Lemma~\ref{lem:initweightmatbd}).  This condition on the initial weight matrices appears in the proof by \citet{liubelkinarxiv}.  The authors \citet{liubelkinarxiv} do not need to explicitly add this assumption because they perform the proof for the case where all the layers have equal width for simplicity of presentation, while stating that the proof generalizes to the case where the layers do not have equal width.  
\paragraph{Exponential dependence on depth}
We note that under the $\tilde{O}$ notation in Theorem~\ref{thm:liuhessbd} there are constants that depend exponentially on the network depth $L$.  For this reason it is essential that the depth $L$ be treated as constant.  We will now briefly explain how the term $\max\{1, R/\sqrt{m}\}^{O(L)}$ arises in the bound in Theorem~\ref{thm:liuhessbd}.  For simplicity assume the network is fully connected at each layer (the same form of argument holds for the other cases).  Let $\xi(\theta) = \max_{l} \frac{1}{\sqrt{m}} \big\|W^{(l)}\big\|_{op}$.  With high probability over the initialization we have that $\xi(\theta_0) = O(1)$ (see Lemma~\ref{lem:initweightmatbd}).  Furthermore for $\theta$ such that $\norm{\theta - \theta_0}_2 \leq R$ we have that $\xi(\theta) \leq \xi(\theta_0) + \frac{R}{\sqrt{m}} = O(\max\{1, R/\sqrt{m}\})$.  It turns out that the features $\alpha^{(l)}$ at each layer $l$ satisfy $\frac{1}{\sqrt{m}} \norm{\alpha^{(l)}}_2 = O(\xi^{O(L)})$.  Well for $\theta$ such that $\norm{\theta - \theta_0}_2 \leq R$ as stated before we have that $\xi(\theta) = O(\max\{1, R/\sqrt{m}\})$.  Consequently for such $\theta$ we get that $\frac{1}{\sqrt{m}} \norm{\alpha^{(l)}}_2 = O(\xi^{O(L)}) = O(\max\{1, R/\sqrt{m}\}^{O(L)})$.  The Hessian bound inherits dependence on the quantity $O(\max\{1, R/\sqrt{m}\}^{O(L)})$ from its dependence the normalized feature $\frac{1}{\sqrt{m}} \norm{\alpha^{(l)}}_2$ norms.
\paragraph{Antisymmetric initialization and the Hessian}
We will now explain how the antisymmetric initialization trick will not hinder us from bounding the Hessian via Theorem~\ref{thm:liuhessbd}.  Let $f(x; \theta)$ denote any model of the form specified in Section~\ref{sec:architecture} where $\theta \in \RR^p$.  Let $\tilde{\theta} = \left[ \begin{smallmatrix}
\theta\\
\theta'
\end{smallmatrix} \right]$
where $\theta, \theta' \in \RR^p$.  Recall the antisymmetric initialization trick defines the model
\[ f_{ASI}(x; \tilde{\theta}) := \frac{1}{\sqrt{2}} f(x; \theta) - \frac{1}{\sqrt{2}} f(x; \theta') \]
which takes the difference of two rescaled copies of the model $f(x; \bullet)$ with parameters $\theta$ and $\theta'$ that are optimized freely.  We then note that the Hessian of $f_{ASI}$ has the block diagonal structure
\[ \nabla_{\tilde{\theta}}^2 f_{ASI}(x; \tilde{\theta}) = \frac{1}{\sqrt{2}} 
\begin{bmatrix}
\nabla_\theta^2 f(x; \theta)& 
0 \\
0 & 
-\nabla_{\theta'}^2 f(x; \theta')
\end{bmatrix} = \frac{1}{\sqrt{2}} \begin{bmatrix}
H(x, \theta) & 
0\\
0 & 
-H(x, \theta')
\end{bmatrix}.
\]
Well then it is not too hard to show that
\[ \norm{\nabla_{\tilde{\theta}}^2 f_{ASI}(x; \tilde{\theta})}_{op} \leq \max \brackets{\norm{H(x, \theta)}_{op}, \norm{H(x, \theta')}_{op}}.  \]
Now recall that the antisymmetric initialization trick initializes $\theta_0 \sim N(0, I)$ then sets
$\tilde{\theta}_0 = \left[\begin{smallmatrix}
\theta_0\\
\theta_0
\end{smallmatrix}\right]$.  Furthermore note that if $\norm{\tilde{\theta} - \tilde{\theta_0}}_2 \leq R$ then $\norm{\theta - \theta_0}_2 \leq R$ and $\norm{\theta' - \theta_0}_2 \leq R$.  Thus if $\theta_0$ is an initialization such that the conclusion of Theorem~\ref{thm:liuhessbd} holds for the model $f(x; \theta)$ then the same conclusion holds for $f_{ASI}(x; \tilde{\theta})$ with initialization $\tilde{\theta}_0$.
\subsection{Definition of the Convolution Operation}
In this subsection we will formally define the convolution operation $*$ introduced in Section~\ref{sec:architecture}.  We use the same convention for the convolution operation as \citet{liubelkinarxiv}.  A convolutional layer of the network has the form
\[ \alpha^{(l)} = \psi_l(\theta^{(l)}, \alpha^{(l - 1)}) = \omega\parens{\frac{1}{\sqrt{m_{{l - 1}}}} W^{(l)} * \alpha^{(l - 1)}}. \]
Here $W^{(l)} \in \RR^{K \times m_{l} \times m_{l - 1}}$ is an order-3 tensor where $K$ denotes the filter size, $m_l$ is the number of output channels, and $m_{l - 1}$ is the number of input channels.  The input $\alpha^{(l - 1)} \in \RR^{m_{l - 1} \times Q}$ is a matrix with $m_{l - 1}$ rows as channels and $Q$ columns as pixels.  The output of the layer $\psi_l$ is of size $\RR^{m_l \times Q}$.  From now on we will drop the superscripts and just denote $W = W^{(l)}$ and $\alpha = \alpha^{(l)}$.  The convolution operation is defined as
\[ (W * \alpha)_{i, q} = \sum_{k = 1}^K \sum_{j = 1}^{m_{l - 1}} W_{k, i, j} \alpha_{j, q + k - \frac{K + 1}{2}}. \]
This can be reformulated as follows.  For each $k \in [K]$ define the matrices $W^{[k]} := W_{k, i, j}$ and $(\alpha^{[k]})_{j, q} := \alpha_{j, q + k - \frac{K + 1}{2}}$.  Then the convolution operation can be rewritten as
\[ (W * \alpha) = \sum_{k = 1}^K W^{[k]} \alpha^{[k]}. \]
Under this reformulation the convolutional layer can be rewritten as
\[ \psi(W, \alpha) = \omega\parens{\sum_{k = 1}^K \frac{1}{\sqrt{m_{l - 1}}} W^{[k]} \alpha^{[k]}}. \]
By treating each $W^{[k]}$ as if it were a weight matrix in a fully connected layer, the convolutional layers can be treated similarly to fully connected layers.  Thus when we refer to weight matrices in the context of a convolutional layer we are referring to the matrices $W^{[k]}$.
\subsection{Technical Lemmas}
This section will cover some miscellaneous technical lemmas that will be of significance later.  The following lemma bounds the operator norm of the weight matrices at initialization.
\begin{lem}\label{lem:initweightmatbd}
Let $f(x; \theta)$ be a neural network of the form specified in Section~\ref{sec:architecture}.  Assume $m \geq d$ and $\max_l \frac{m_l}{m} \leq A$.  Then with probability at least $1 - C\exp(-c m)$ over the initialization $\theta_0$ each weight matrix $W_0$ at initialization satisfies
\[ \frac{1}{\sqrt{m}} \norm{W_0} \leq 2 \sqrt{A} + 1. \]
The constant $C > 0$ depends on the architecture but is independent of the width $m$.
\end{lem}
\begin{proof}
Fix a weight matrix $W \in \RR^{m_l \times m_{l - 1}}$ in the model.  Following \citet[Corollary 5.35]{vershynin2011introduction} we have with probability at least $1 - 2\exp(-t^2 / 2)$ over the initialization
\[ \norm{W_0}_{op} \leq \sqrt{m_l} + \sqrt{m_{l - 1}} + t \]
and thus
\[ \frac{1}{\sqrt{m}} \norm{W_0^{(l)}}_{op} \leq \frac{\sqrt{m_l}}{\sqrt{m}} + \frac{\sqrt{m_{l - 1}}}{\sqrt{m}} + \frac{t}{\sqrt{m}} \leq 2\sqrt{A} + \frac{t}{\sqrt{m}}. \]
Thus by setting $t = \sqrt{m}$ and taking the union bound over all weight matrices in the model (which depends on the architecture) we get the desired result.
\end{proof}
We now state for reference the following lemma which follows from the proof in \citep{liubelkinarxiv}.
\begin{lem}\label{lem:graddeterministicbd}
Let $R \geq 1$ and let $f(x; \theta)$ be a neural network of the form specified in Section~\ref{sec:architecture}.  If $\theta_0$ is an initialization such that each weight matrix $W_0$ satisfies
$\frac{1}{\sqrt{m}} \big\|W_0^{(l)}\big\|_2 = O(1)$ then
\[ \sup_{x \in X} \sup_{\theta \in \overline{B}(\theta_0, R)} \norm{\nabla_\theta f(x; \theta)}_2 = O\parens{\max\braces{1, \frac{R}{\sqrt{m}}}^{O(L)}}. \]
In particular if $\sqrt{m} \geq R$ then
\[ \sup_{x \in X} \sup_{\theta \in \overline{B}(\theta_0, R)} \norm{\nabla_\theta f(x; \theta)}_2 = O\parens{1}. \]
\end{lem}
As a consequence of the previous lemma we get the following high probability bound on the gradients norm $\norm{\nabla_\theta f(x; \theta)}_2$.
\begin{lem}\label{lem:gradbd}
Let $R \geq 1$ and let $f(x; \theta)$ be a neural network of the form specified in Section~\ref{sec:architecture}.  Assume that $m \geq d$, $\max_l \frac{m_l}{m} = O(1)$, and $\sqrt{m} \geq R$.  Then with probability at least $1 - C\exp(-c m)$ over the initialization $\theta_0$ we have that
\[ \sup_{x \in X} \sup_{\theta \in \overline{B}(\theta_0, R)} \norm{\nabla_\theta f(x; \theta)}_2 = O(1). \]
The constant $C > 0$ depends on the architecture but is independent of the width $m$
\end{lem}
\begin{proof}
This follows immediately from Lemma~\ref{lem:initweightmatbd} and Lemma~\ref{lem:graddeterministicbd}.
\end{proof}
The following lemma bounds the kernel deviations $K^\theta - K^{\theta_0}$ in terms of the network Hessian.
\begin{lem}\label{lem:kernelbdbyhess}
Let $S = \{z_1, \ldots, z_k\} \subset X$.  Let $B = \sup_{x \in X} \sup_{\theta \in \overline{B}(\theta_0, R)} \norm{\nabla_\theta f(x; \theta)}$ and let $H_{max} = \max_{z \in S} \sup_{\theta \in \overline{B}(\theta_0, R)} \norm{H(z, \theta)}_{op}$.  Then for $\theta \in \overline{B}(\theta_0, R)$
\[ \max_{i, j \in [k]} |K^\theta(z_i, z_j) - K^{\theta_0}(z_i, z_j)| \leq 2 BH_{max} R. \]
\end{lem}
\begin{proof}
We have that
\begin{gather*}
|K^\theta(z_i, z_j) - K^{\theta_0}(z_i, z_j)| \\ \leq
\norm{\nabla_\theta f(z_i; \theta)} \norm{\nabla_\theta f(z_j; \theta) - \nabla_\theta f(z_j; \theta_0)} + \norm{\nabla_\theta f(z_i;\theta) - \nabla_\theta f(z_i; \theta_0)} \norm{\nabla_\theta f(z_j; \theta_0)} \\ \leq 2 B H_{max} R.
\end{gather*}
Here we have used the fact that
\begin{gather*}
\norm{\nabla_\theta f(z_i; \theta) - \nabla_\theta f(z_i; \theta_0)}_2 = \norm{\int_0^1 H(z_i, s \theta + (1 - s) \theta_0) (\theta - \theta_0) ds}_2 \\
\leq \int_0^1 \norm{H(z_i, s \theta + (1 - s) \theta_0)}_{op} \norm{\theta - \theta_0}_2 \leq H_{max} R.
\end{gather*}
\end{proof}
The following lemma provides a trivial bound on $\norm{\theta_t - \theta_0}_2$.
\begin{lem}\label{lem:aprioriparambd}
\[\norm{\theta_t - \theta_0}_2 \leq \frac{\sqrt{t}}{\sqrt{2}} \norm{\hat{r}_0}_{\RR^n} \leq \frac{\sqrt{t}}{\sqrt{2}} \norm{f^*}_{L^\infty(X, \rho)}. \]
\end{lem}
\begin{proof}
\begin{gather*}
\norm{\theta_t - \theta_0}_2 \leq \int_0^t \norm{\partial_s \theta_s}_2 ds =
\int_0^t \norm{\partial_\theta L(\theta_s)}_2 ds \leq \sqrt{t} \brackets{\int_0^t \norm{\partial_\theta L(\theta_s)}_2^2 ds}^{1/2}  \\
= \sqrt{t} \brackets{\int_0^t - \partial_s L(\theta_s) ds}^{1/2} =
\sqrt{t} \brackets{L(\theta_0) - L(\theta_t)}^{1/2} \leq
\sqrt{t} \brackets{L(\theta_0)}^{1/2} = \frac{\sqrt{t}}{\sqrt{2}} \norm{\hat{r}_0}_{\RR^n} \\
\leq \frac{\sqrt{t}}{\sqrt{2}} \norm{f^*}_{L^\infty(X, \rho)}
\end{gather*}
where the second inequality above follows from the Cauchy-Schwarz inequality and the final inequality follows from the fact that $\norm{\hat{r}_0}_{\RR^n} = \norm{y}_{\RR^n} \leq \norm{f^*}_{L^\infty(X, \rho)}$ from the antisymmetric initialization.
\end{proof}

\section{Convergence of the Operators}
\label{sec:conv_operators}
Throughout this section $K(x, x')$ will be a fixed continuous, symmetric, positive definite kernel.  We will let $\kappa := \max_{x \in X} K(x, x)$.  We note that since $K$ is continuous and $X$ is compact we have that $\kappa < \infty$.  We will thus treat $\kappa$ as a constant.  We also note that since $K$ is a kernel for any $x, x' \in X$ we have the inequality $K(x, x') \leq \sqrt{K(x, x)} \sqrt{K(x', x')} \leq \kappa$.  
\par
We will let $K^\theta(x, x') = \langle \nabla_\theta f(x; \theta), \nabla_\theta f(x'; \theta) \rangle$ denote the NTK for a specific parameter $\theta$.  In this section $\theta_0$ will be treated as fixed.  We will show that for fixed $\theta_0$ we have bounds on $\norm{(T_K - T_n^s)r_s}_{L^2(X, \rho)}$ that hold with high probability over the sampling of $S = (x_1, \ldots, x_n)$.  By the Fubini-Tonelli theorem this suffices to get bounds that hold with high probability over the parameter initialization $\theta_0 \sim \mu$ and data sampling $S \sim \rho^{\otimes n}$ as long as one makes sure that the appropriate events are measureable on the product space.  Fortunately, due to the continuity of $K^\theta(x, x')$ and $H(x, \theta)$ with respect to $x, x'$ and $\theta$ we can avoid such issues and we thus will not address measureability line-by-line.
\par In this section we will bound $\norm{(T_K - T_n^s) r_s}_{L^2(X, \rho)}$ for all $s$ such that $\norm{\theta_s - \theta_0}_2 \leq R$.  This will be done by bounding $\norm{(T_K - T_n)r_s}_{L^2(X, \rho)}$ and $\norm{(T_n - T_n^s)r_s}_{L^2(X, \rho)}$ separately.  At a high level $\norm{(T_n - T_n^s) r_s}_{L^2(X, \rho)}$ will be small whenever $K_0 - K_s$ is small.  On the other hand $\norm{(T_K - T_n)r_s}_{L^2(X, \rho)}$ will be small whenever $n$ is large enough relative to the complexity of the function class $\{ f(x; \theta) : \theta \in \overline{B}(\theta_0, R)\}$.  If $\sup_{\theta \in \overline{B}(\theta_0, R)}\norm{H(x, \theta)}_2$ was uniformly small over $x$ then the kernel deviations $K_0 - K_s$ would be bounded and the complexity of $\{ f(x; \theta) : \theta \in \overline{B}(\theta_0, R)\}$ would be controlled by the complexity of the linearized model $f_{lin}(x; \theta) = \langle \nabla_\theta f(x; \theta_0), \theta - \theta_0 \rangle$.  However, Theorem \ref{thm:liuhessbd} only gives us the ability to bound $\norm{H(x, \theta)}$ for finitely many values of $x$.  For this reason we will need to do somewhat elaborate gymnastics using Rademacher complexity to form estimates that only require the evaluation of $\sup_{\theta \in \overline{B}(\theta_0, R)}\norm{H(x, \theta)}$ over finitely many values of $x$.

Let $\mathcal{F}$ denote some family of real valued functions and let $S = (z_1, \ldots, z_k)$ be a finite point set.  We define 
\[ \mathcal{F}_{|S} = \{ (g(z_1), \ldots, g(z_k)) : g \in \mathcal{F}\} \]
to be the set of all vectors in $\RR^k$ formed by restricting a function in $\mathcal{F}$ to the point set $S$.  Now let $\epsilon \in \RR^k$ be a vector with entries that are i.i.d.\ Rademacher random variables, i.e. $\epsilon_i \sim \text{Unif}\{+1, -1\}$.  We define the (unnormalized) Rademacher complexity of $\mathcal{F}_{|S}$.
\[ URad(\mathcal{F}_{|S}) := \EE_{\epsilon} \sup_{v \in \mathcal{F}_{|S}} \langle v, \epsilon \rangle = \EE \sup_{g \in \mathcal{F}} \sum_{i = 1}^k \epsilon_i g(x_i). \]
We will use the following classic result, see e.g.\ \citet[Theorem 13.1]{mjt_dlt}
\begin{theo}\label{thm:mjtrademacher}
Let $\mathcal{F}$ be given with $g(z) \in [a, b]$ a.s.\ for all $g \in \mathcal{F}$.  Then with probability at least $1 - \delta$ over the sampling of $z_1, \ldots, z_n$
\[ \sup_{g \in \mathcal{F}}\brackets{\EE[g(Z)] - \frac{1}{n} \sum_{i = 1}^n g(z_i)} \leq \frac{2}{n} URad(\mathcal{F}_{|S}) + 3(b - a) \sqrt{\frac{\log(2/\delta)}{2n}}. \]
\end{theo}
\noindent We will also make use of the following lemma which is also classic, see e.g. \citet[Lemma 13.3]{mjt_dlt}
\begin{lem}\label{lem:mjtliplem}
Let $\ell: \RR^n \rightarrow \RR^n$ be a vector of univariate $C$-lipschitz functions. Then $URad((\ell \circ \mathcal{F})_{|S}) \leq C \cdot URad(\mathcal{F}_{|S})$.
\end{lem}
\noindent Using this we will now prove the following technical lemma.  For the purpose of this lemma $x_1, \ldots, x_n$ will be treated as fixed and the randomness will be over a ghost sample $S' = (x_1', \ldots, x_n')$.
\begin{lem}\label{lem:kerneldevghostsample}
Let $R \geq 1$ and $B = \sup_{x \in X} \sup_{\theta \in \overline{B}(\theta_0, R)} \norm{\nabla_\theta f(x, \theta)}_2.$  Consider $x_1, \ldots, x_n \in X$ to be fixed.  Then let 
\[ \mathcal{F} = \{x \mapsto \frac{1}{n} \sum_{i = 1}^n |K^\theta(x, x_i) - K^{\theta_0}(x, x_i)|^2 : \theta \in \overline{B}(\theta_0, R) \}. \]
Let $x_1', \ldots, x_n'$ be sampled i.i.d.\ from $\rho$.  Let $S = (x_1, \ldots, x_n)$ and $S' = \{x_1', \ldots, x_n'\}$ and define
\[ H_{max} := \max_{z \in S \cup S'} \sup_{\theta \in \overline{B}(\theta_0, R)} \norm{H(z, \theta)}_{op} \]
Then with probability at least $1 - \delta$ over the sampling of $x_1', \ldots, x_n'$ we have that every $g \in \mathcal{F}$ satisfies
\[ \EE_{x \sim \rho}[g(x)] \leq 12 B^2 H_{max}^2 R^2 + 12 B^4 \sqrt{\frac{\log(2/\delta)}{2n}}.   \]
\end{lem}
\begin{proof}
We note that for $\theta \in \overline{B}(\theta_0, R)$
\[ |K^\theta(x, x_i) - K^{\theta_0}(x, x_i)|^2  \leq [|K^\theta(x, x_i)| + |K^{\theta_0}(x, x_i)|]^2 \leq [2B^2]^2 = 4 B^4. \]
Therefore for all $g \in \mathcal{F}$ we have that $g(x) \in [0, 4B^4]$ a.s.  Then by Theorem \ref{thm:mjtrademacher} we have with probability at least $1 - \delta$ over the sampling of $S' = \{x_1', \ldots, x_n'\}$
\[ \sup_{g \in \mathcal{F}}\brackets{\EE_{x \sim \rho}[g(x)] - \frac{1}{n} \sum_{i = 1}^n g(x_i')} \leq \frac{2}{n} URad(\mathcal{F}_{|S'}) + 12 B^4 \sqrt{\frac{\log(2/\delta)}{2n}}. \]
Then we note that for any $z, z' \in S \cup S'$ we by Lemma \ref{lem:kernelbdbyhess} that $\theta \in \overline{B}(\theta_0, R)$ implies
\[ |K^\theta(z, z') - K^{\theta_0}(z, z')| \leq 2B H_{max}R. \]
It follows that for any member of $\mathcal{F}_{| S \cup S'}$ is bounded in infinity norm by $4B^2 H_{max}^2 R^2$.  Thus for any $g \in \mathcal{F}$ we have that
\[ \frac{1}{n} \sum_{i = 1}^n g(x_i') \leq 4B^2 H_{max}^2 R^2 \]
and
\[ \frac{1}{n} URad(\mathcal{F}_{|S'}) \leq 4B^2 H_{max}^2 R^2. \]

Therefore for any $g \in \mathcal{F}$ we have that
\begin{gather*}
\EE_{x \sim \rho}[g(x)] \leq \frac{1}{n} \sum_{i = 1}^n g(x_i') + \frac{2}{n} URad(\mathcal{F}_{|S'}) + 12 B^4 \sqrt{\frac{\log(2/\delta)}{2n}} \\
\leq 12 B^2 H_{max}^2 R^2 + 12 B^4 \sqrt{\frac{\log(2/\delta)}{2n}}.    
\end{gather*}

\end{proof}
\noindent Using the previous lemma we can now bound $\norm{(T_n - T_n^t)r_t}_{L^2(X, \rho)}$.
\begin{lem}\label{lem:maintnminustntbd}
Let $R \geq 1$ and $B = \sup_{x \in X} \sup_{\theta \in \overline{B}(\theta_0, R)} \norm{\nabla_\theta f(x, \theta)}_2.$  Let $S = (x_1, \ldots, x_n)$ and $S' = (x_1', \ldots, x_n')$ be two independent sequences of i.i.d.\ samples from $\rho$.  Define
\[ H_{max} := \max_{z \in S \cup S'} \sup_{\theta \in \overline{B}(\theta_0, R)} \norm{H(z, \theta)}_{op}. \]
Then with probability at least $1 - \delta$ over the sampling of $S$ and $S'$ we have that for any $\theta_t$ such that $\norm{\theta_t - \theta_0}_2 \leq R$,
\[\norm{(T_n - T_n^t) r_t}_{L^2(X, \rho)}^2 \leq  2 \norm{f^*}_{L^\infty(X, \rho)}^2 \brackets{\norm{K - K_0}_{L^2(X^2, \rho \otimes \rho)}^2 + 12B^2 H_{max}^2 R^2 + \tilde{O}\parens{\frac{B^4}{\sqrt{n}}}}. \]
\end{lem}
\begin{proof}
We note that
\begin{gather*}
\abs{(T_n - T_n^t)r_t(x)} = \abs{\frac{1}{n} \sum_{i = 1}^n [K(x, x_i) - K_t(x, x_i)] r_t(x_i)} \\
\leq \norm{\hat{r}_t}_{\RR^n} \brackets{\frac{1}{n} \sum_{i = 1}^n |K(x, x_i) - K_t(x, x_i)|^2}^{1/2}
\leq \norm{\hat{r}_0}_{\RR^n} \brackets{\frac{1}{n} \sum_{i = 1}^n |K(x, x_i) - K_t(x, x_i)|^2}^{1/2}
\end{gather*}
where we have used the property $\norm{\hat{r}_t}_{\RR^n} \leq \norm{\hat{r}_0}_{\RR^n}$ from gradient flow.  Well from the inequality $(a + b)^2 \leq 2(a^2 + b^2)$ we have that
\begin{gather*}
\frac{1}{n}\sum_{i = 1}^n |K(x, x_i) - K_t(x, x_i)|^2  \\
\leq 
\frac{2}{n} \sum_{i = 1}^n |K(x, x_i) - K_0(x, x_i)|^2 + \frac{2}{n} \sum_{i = 1}^n |K_0(x, x_i) - K_t(x, x_i)|^2.
\end{gather*}
For conciseness let
\[ h_1(x) := \frac{1}{n} \sum_{i = 1}^n |K(x, x_i) - K_0(x, x_i)|^2 \]
\[ h_2^t(x) := \frac{1}{n} \sum_{i = 1}^n |K_0(x, x_i) - K_t(x, x_i)|^2. \]
Then by the above we have that
\[ \norm{(T_n - T_n^t) r_t}_{L^2(X, \rho)}^2 \leq 2\norm{\hat{r}_0}_{\RR^n}^2 \brackets{\EE_{x \sim \rho}[h_1(x)] + \EE_{x \sim \rho}[h_2^t(x)]}.  \]
Well we note that $|K(x, x')| \leq \kappa$ and $|K_0(x, x')| \leq B^2$ uniformly over $x, x'$.  Now consider the random variables $Z_i := \norm{K(\bullet, x_i) - K_0(\bullet, x_i)}_{L^2(X, \rho)}^2$ where the randomness is over the sampling of $x_i$.  Then we have that $|Z_i| \leq [\kappa + B^2]^2$ a.s.  Thus by Hoeffding's inequality we have that
\[ \PP\parens{\frac{1}{n} \sum_{i = 1}^n Z_i - \EE_{x_1 \sim \rho}[Z_1] > s} \leq  \exp\parens{\frac{-n s^2}{2[\kappa + B^2]^4}}. \]
Thus with probability at least $1 - \delta$ over the sampling of $x_1, \ldots, x_n$
\begin{equation}\label{eq:ziave}
\frac{1}{n} \sum_{i = 1}^n Z_i \leq \EE_{x_1 \sim \rho}[Z_1] + \frac{\sqrt{2} [\kappa + B^2]^2 \sqrt{\log(1/\delta)}}{\sqrt{n}}.
\end{equation}
Now note that 
\[ \frac{1}{n} \sum_{i = 1}^n Z_i = \EE_{x \sim \rho}[h_1(x)] \qquad \EE_{x_1 \sim \rho}[Z_1] = \norm{K - K_0}_{L^2(X^2, \rho \otimes \rho)}^2. \]
Thus whenever \eqref{eq:ziave} holds we have that
\begin{gather*}
\EE_{x \sim \rho}[h_1(x)] \leq \norm{K - K_0}_{L^2(X^2, \rho \otimes \rho)}^2 + \frac{\sqrt{2} [\kappa + B^2]^2 \sqrt{\log(1/\delta)}}{\sqrt{n}} \\
= \norm{K - K_0}_{L^2(X^2, \rho \otimes \rho)}^2 + \tilde{O}\parens{\frac{B^4}{\sqrt{n}}}.
\end{gather*}
On the other hand we have by Lemma \ref{lem:kerneldevghostsample} for any fixed $x_1, \ldots, x_n$ that with probability $1 - \delta$ over the sampling of $x_1', \ldots, x_n'$ i.i.d.\ from $\rho$ we have that for all $\theta \in \overline{B}(\theta_0, R)$
\begin{equation}\label{eq:radbd2}
\EE_{x \sim \rho }\brackets{\frac{1}{n} \sum_{i = 1}^n |K^{\theta}(x, x_i) - K^{\theta_0}(x, x_i)|^2} \leq 12 B^2 H_{max}^2 R^2 + 12 B^4 \sqrt{\frac{\log(2/\delta)}{2n}}.    
\end{equation}
Whenever the above holds we have that for any $\theta_t$ such that $\norm{\theta_t - \theta_0}_2 \leq R$ we have that

\[ \EE_{x \sim \rho}[h_2^t(x)] \leq 12 B^2 H_{max}^2 R^2 + 12 B^4 \sqrt{\frac{\log(2/\delta)}{2n}} = 12 B^2 H_{max}^2 R^2 + \tilde{O}\parens{\frac{B^4}{\sqrt{n}}}.  \]
Thus combining these together we have with probability at least $(1 - \delta)^2 \geq 1 - 2\delta$ over the sampling of $x_1, \ldots, x_n, x_1', \ldots, x_n'$ that Equations \eqref{eq:ziave} and \eqref{eq:radbd2} hold simultaneously for all $\theta \in \overline{B}(\theta_0, R)$.  In such a case we have that for all $\theta_t$ such that $\norm{\theta_t - \theta_0}_2 \leq R$ that
\begin{gather*}
\EE_{x \sim \rho}[h_1(x)] + \EE_{x \sim \rho}[h_2^t(x)] \leq \norm{K - K_0}_{L^2(X^2, \rho \otimes \rho)}^2 + 12 B^2 H_{max}^2 R^2 + \tilde{O}\parens{\frac{B^4}{\sqrt{n}}}.
\end{gather*}
Well then
\begin{gather*}
\norm{(T_n - T_n^t) r_t}_{L^2(X, \rho)}^2 \leq 2\norm{\hat{r}_0}_{\RR^n}^2 \brackets{\EE_{x \sim \rho}[h_1(x)] + \EE_{x \sim \rho}[h_2^t(x)]}\\
\leq  2 \norm{\hat{r}_0}_{\RR^n}^2\brackets{\norm{K - K_0}_{L^2(X^2, \rho \otimes \rho)}^2 + 12 B^2 H_{max}^2 R^2 + \tilde{O}\parens{\frac{B^4}{\sqrt{n}}}} \\
\leq 2 \norm{f^*}_{L^\infty(X, \rho)}^2 \brackets{\norm{K - K_0}_{L^2(X^2, \rho \otimes \rho)}^2 + 12B^2 H_{max}^2 R^2 + \tilde{O}\parens{\frac{B^4}{\sqrt{n}}}}.
\end{gather*}
In the last line above we have used the fact that $\norm{\hat{r}_0}_{\RR^n} = \norm{y}_{\RR^n} \leq \norm{f^*}_{L^\infty(X, \rho)}$ from the antisymmetric initialization.  The desired result follows after replacing $\delta$ with $\delta / 2$ in the previous argument.
\end{proof}
From Lemma~\ref{lem:maintnminustntbd} we get the following corollary.
\begin{cor}\label{cor:tnminustnt}
Let $R \geq 1$, $B = \sup_{x \in X} \sup_{\theta \in \overline{B}(\theta_0, R)} \norm{\nabla_\theta f(x, \theta)}_2.$  Let $S = (x_1, \ldots, x_n)$ and $S' = (x_1', \ldots, x_n')$ be two independent sequences of i.i.d.\ samples from $\rho$.  Define
\[ H_{max} := \max_{z \in S \cup S'} \sup_{\theta \in \overline{B}(\theta_0, R)} \norm{H(z, \theta)}_{op}. \]
Then with probability at least $1 - \delta$ over the sampling of $S$ and $S'$ we have that for any $\theta_t$ such that $\norm{\theta_t - \theta_0}_2 \leq R$
\[\norm{(T_n - T_n^t) r_t}_{L^2(X, \rho)}^2 \leq 2 \norm{f^*}_{L^\infty(X, \rho)}^2 \norm{K - K_0}_{L^2(X^2, \rho \otimes \rho)}^2 + \epsilon \]
provided that $B = \tilde{O}(1)$, $H_{max} = \tilde{O}(\epsilon^{1/2} / R)$ and $n = \tilde{\Omega}(\epsilon^{-2})$.
\end{cor}
\begin{proof}
We have by Lemma \ref{lem:maintnminustntbd} with probability at least $1 - \delta$ over the sampling of $S$, $S'$
\[\norm{(T_n - T_n^t) r_t}_{L^2(X, \rho)}^2 \leq 2 \norm{f^*}_{L^\infty(X, \rho)}^2 \brackets{\norm{K - K_0}_{L^2(X^2, \rho \otimes \rho)}^2 + 12B^2 H_{max}^2 R^2 + \tilde{O}\parens{\frac{B^4}{\sqrt{n}}}}.\]
Thus if $B = \tilde{O}(1)$ then $H_{max} = \tilde{O}(\epsilon^{1/2} / R)$ and $n = \tilde{\Omega}(\epsilon^{-2})$ is sufficient to ensure that
\[\norm{(T_n - T_n^t) r_t}_{L^2(X, \rho)}^2 \leq 2 \norm{f^*}_{L^\infty(X, \rho)}^2 \norm{K - K_0}_{L^2(X^2, \rho \otimes \rho)}^2 + \epsilon. \]
\end{proof}
\noindent Now we will begin the work to bound $\norm{(T_K - T_n) r_s}_{L^2(X, \rho)}$.  The following technical lemma bounds the Rademacher complexity of the difference between the network $f(x; \theta)$ and the linearization $f_{lin}(x; \theta) = \langle \nabla_\theta f(x; \theta_0), \theta - \theta_0 \rangle$ in terms of the Hessian norm for finitely many values $z \in X$.
\begin{lem}\label{lem:fminusflinrad}
Let $R \geq 1$, $\mathcal{F} = \{x \mapsto f(x; \theta) - f_{lin}(x; \theta) : \theta \in \overline{B}(\theta_0, R)\}$,  $B = \sup_{x \in X}  \sup_{\theta \in \overline{B}(\theta_0, R)} \norm{\nabla_\theta f(x; \theta)}$, and let $S = (z_1. \ldots, z_n) \subset X$.  Furthermore let
\[ H_{max} := \max_{z \in S} \sup_{\theta \in \overline{B}(\theta_0, R)} \norm{H(z, \theta)}_{op}. \]  Then
\[ \sup_{g \in \mathcal{F}} \norm{g}_{L^\infty(X, \rho)} \leq 2BR \]
and
\[ \sup_{g \in \mathcal{F}} \max_{z \in S} |g(z)| \leq \frac{1}{2} R^2 H_{max}. \]
In particular
\[ \frac{1}{n} URad(\mathcal{(F \cup - F)}_{|S}) \leq \frac{1}{2} R^2 H_{max}. \]
\end{lem}
\begin{proof}
We note that
\[ |f(x; \theta) - f_{lin}(x; \theta)| \leq |f(x; \theta)| + |f_{lin}(x; \theta)|. \]
Well then using the fact that $f(\bullet; \theta_0) = 0$ from the antisymmetric initialization we get
\begin{gather*}
|f(x; \theta)| = |f(x; \theta) - f(x; \theta_0)| = \abs{\int_0^1 \langle \nabla_\theta f(x; \theta s + (1 - s) \theta_0), \theta - \theta_0 \rangle ds} \\
\leq \int_0^1 |\langle \nabla_\theta f(x; \theta s + (1 - s) \theta_0), \theta - \theta_0 \rangle| \leq B \norm{\theta - \theta_0} \leq BR.    
\end{gather*}

On the other hand
\[ |f_{lin}(x; \theta)| = |\langle \nabla_\theta f(x; \theta_0), \theta - \theta_0 \rangle| \leq \norm{\nabla_\theta f(x; \theta_0)}_2 \norm{\theta - \theta_0}_2 \leq B R. \]
Thus
\[\sup_{\theta \in \overline{B}(\theta_0, R)} \norm{f(\bullet; \theta) - f_{lin}(\bullet;\theta)}_{L^\infty(X, \rho)} \leq 2BR \]
and the first conclusion follows.  Furthermore by the Lagrange form of the remainder in Taylor's theorem we have for $z \in S$
\begin{gather*}
|f(z; \theta) - f_{lin}(z;\theta)| = \abs{(\theta - \theta_0)^T \frac{H(z, \xi)}{2} (\theta - \theta_0)} 
\leq \frac{1}{2} \norm{\theta - \theta_0}_2^2 \norm{H(z, \xi)}_{op}
\end{gather*}
where $\xi$ is some point on the line between $\theta$ and $\theta_0$.  Thus if we set 
\[ H_{max} := \max_{z \in S} \sup_{\theta \in \overline{B}(\theta_0, R)} \norm{H(z, \theta)}_{op} \]
we have that 
\[|f(z; \theta) - f_{lin}(z;\theta)| \leq \frac{1}{2} R^2 H_{max}\]
for all $\theta \in \overline{B}(\theta_0, R)$.  Therefore $\frac{1}{n} URad(\mathcal{(F \cup - F)}_{|S}) \leq \frac{1}{2} R^2 H_{max}$ and the desired result follows.
\end{proof}
We now introduce another technical lemma that provides Rademacher complexity and $L^\infty$ norm bounds for the linear model $x \mapsto \langle \nabla_\theta f(x; \theta_0), \theta \rangle$.
\begin{lem}\label{lem:linradbd}
Let $R \geq 1$, $\mathcal{F} = \{x \mapsto \langle \nabla_\theta f(x; \theta_0), \theta \rangle : \norm{\theta}_2 \leq 2R \}$.  Let $B =  \sup_{x \in X} \sup_{\theta \in \overline{B}(\theta_0, R)} \norm{\nabla_\theta f(x; \theta)}$.  Then
\[ \sup_{g \in \mathcal{F}} \norm{g}_{L^\infty(X, \rho)} \leq 2BR \]
and
\[ \frac{1}{n} URad(\mathcal{F}_{|S}) \leq \frac{2BR}{\sqrt{n}}. \]
\end{lem}
\begin{proof}
By Cauchy-Schwarz
\[|\langle \nabla_\theta f(x; \theta_0), \theta \rangle| \leq 2BR \]
and thus $\norm{g}_{L^\infty(X, \rho)} \leq 2BR$ for all $g \in \mathcal{F}$.  Now let $\epsilon \in \RR^n$ be a vector with i.i.d Rademacher entries $\epsilon_i \sim \text{Unif}\{+1, -1\}$.  Then as was shown by \citet[Lemma 22]{bartlettmendelson2003}
\begin{align*}
\EE_{\epsilon}\brackets{\sup_{\theta \in \overline{B}(\theta_0, 2R)} \sum_{i = 1}^n \epsilon_i \langle \nabla_\theta f(x_i, \theta_0), \theta \rangle} &= 2R \EE_{\epsilon} \norm{\sum_{i = 1}^n \epsilon_i \nabla_\theta f(x_i; \theta_0)}_2 \\
&\leq
2R \brackets{\EE_{\epsilon} \norm{\sum_{i = 1}^n \epsilon_i \nabla_\theta f(x_i; \theta_0)}_2^2}^{1/2} \\
&= 2R \brackets{\EE_{\epsilon}\brackets{\sum_{1 \leq i, j \leq n} \epsilon_i \epsilon_j \langle \nabla_\theta f(x_i; \theta_0), \nabla_\theta f(x_j; \theta_0)\rangle}}^{1/2} \\
&= 2R \sqrt{\sum_{i = 1}^n K^{\theta_0}(x_i, x_i)} \\
&\leq 2RB \sqrt{n}.
\end{align*}
where the first inequality above is an application of Jensen's inequality.  The Rademacher complexity bound then follows from the bound above.
\end{proof}
The following lemma compares the $L^2(X, \rho)$ norm to that of its empirical counterpart $L^2(X, \widehat{\rho})$ for the function classes discussed in Lemmas \ref{lem:fminusflinrad} and \ref{lem:linradbd}.
\begin{lem}\label{lem:combinedradbd}
Let $R \geq 1$, $\mathcal{F}_1 = \{x \mapsto f(x; \theta) - f_{lin}(x; \theta) : \theta \in \overline{B}(\theta_0, R) \}$, $\mathcal{F}_2 = \{x \mapsto \langle \nabla_\theta f(x; \theta_0), \theta \rangle : \norm{\theta}_2 \leq 2R \}$, and $B = \sup_{x \in X}  \sup_{\theta \in \overline{B}(\theta_0, R)} \norm{\nabla_\theta f(x; \theta)}$.
Then with probability at least $1 - \delta$ over the sampling of $S = (x_1, \ldots, x_n)$
\[ \sup_{g \in  \mathcal{F}_1 \cup \mathcal{F}_2} \abs{\norm{g}_{L^2(X, \rho)}^2 - \norm{g}_{L^2(X, \widehat{\rho})}^2} \leq 4BR^3 H_{max} +  \tilde{O}\parens{\frac{B^2 R^2}{\sqrt{n}}}. \]
where $\hat{\rho} = \frac{1}{n} \sum_{i = 1}^n \delta_{x_i}$ is the empirical measure induced by $x_1, \ldots, x_n$ and
\[ H_{max} := \max_{z \in S} \sup_{\theta \in \overline{B}(\theta_0, R)} \norm{H(z, \theta)}_{op}. \]
\end{lem}
\begin{proof}
Let $\mathcal{F} = \{ |g|^2 : g \in \mathcal{F}_1 \cup \mathcal{F}_2 \}$.  Note that by Lemmas \ref{lem:fminusflinrad} and \ref{lem:linradbd} we have that
for $g \in \mathcal{F}_1 \cup \mathcal{F}_2$ that $\norm{g}_{L^\infty(X, \rho)} \leq 2BR$.  Thus every $g \in \mathcal{F}$ satisfies $g(x) \in [0, 4B^2 R^2]$ a.s.  Well then by Theorem \ref{thm:mjtrademacher} we have with probability at least $1 - \delta$ over the sampling of $S = (x_1, \ldots, x_n)$ that
\[ \sup_{g \in \mathcal{F}}\brackets{\EE_{x \sim \rho}[g(x)] - \frac{1}{n} \sum_{i = 1}^n g(x_i)} \leq \frac{2}{n} URad(\mathcal{F}_{|S}) + 12 B^2 R^2 \sqrt{\frac{\log(2/\delta)}{2n}}. \]
Well note that $x^2$ is $4BR$ Lipschitz on the interval $[-2BR, 2BR]$.  Then by Lemma \ref{lem:mjtliplem} we have that
\[ URad(\mathcal{F}_{|S}) \leq 4BR \cdot URad((\mathcal{F}_1 \cup \mathcal{F}_2)_{|S}). \]
Well then we have that
\[ URad((\mathcal{F}_1 \cup \mathcal{F}_2)_{|S}) \leq URad((\mathcal{F}_1 \cup - \mathcal{F}_1 \cup \mathcal{F}_2)_{|S}) \leq URad((\mathcal{F}_1 \cup -\mathcal{F}_1)_{|S}) + URad((\mathcal{F}_2)_{|S}) \]
where we have used the property that if $A, A'$ are vector classes such that $\sup_{u \in A} \langle \epsilon, u \rangle \geq 0$ and $\sup_{u \in A'} \langle \epsilon, u \rangle \geq 0$ for all $\epsilon \in \{1, -1\}^{n}$ then $URad(A \cup A') \leq URad(A) + URad(A')$.  Well by Lemma \ref{lem:fminusflinrad} we have that
\[ \frac{1}{n} URad((\mathcal{F}_1 \cup - \mathcal{F}_1)_{|S}) \leq \frac{1}{2} R^2 H_{max}. \]
On the other hand by Lemma \ref{lem:linradbd} we have that
\[ \frac{1}{n} URad((\mathcal{F}_2)_{|S}) \leq \frac{2BR}{\sqrt{n}}. \]
Therefore combining these two bounds we get that
\begin{gather*}
\frac{1}{n} URad((\mathcal{F}_1 \cup \mathcal{F}_2)_{|S}) \leq \frac{1}{2} R^2 H_{max} + \frac{2BR}{\sqrt{n}}
\end{gather*}
and thus
\[ \frac{1}{n} URad(\mathcal{F}_{|S}) \leq \frac{4BR}{n} \cdot URad((\mathcal{F}_1 \cup \mathcal{F}_2)_{|S}) \leq 4BR \brackets{\frac{1}{2} R^2 H_{max} + \frac{2BR}{\sqrt{n}}} . \]
Therefore by putting everything together we have that
\begin{gather*}
\sup_{g \in \mathcal{F}}\brackets{\EE_{x \sim \rho}[g(x)] - \frac{1}{n} \sum_{i = 1}^n g(x_i)} \leq 8BR \brackets{\frac{1}{2} R^2 H_{max} + \frac{2BR}{\sqrt{n}}} + 12 B^2 R^2 \sqrt{\frac{\log(2/\delta)}{2n}} \\
= 4BR^3 H_{max} + \frac{16 B^2 R^2}{\sqrt{n}} + 12 B^2 R^2 \sqrt{\frac{\log(2/\delta)}{2n}}.
\end{gather*}
By repeating the same argument for the class $-\mathcal{F}$ and taking a union bound we have with probability at least $1 - 2 \delta$ that
\[ \sup_{g \in \mathcal{F}}\abs{\EE_{x \sim \rho}[g(x)] - \frac{1}{n} \sum_{i = 1}^n g(x_i)} \leq 4BR^3 H_{max} + \frac{16 B^2 R^2}{\sqrt{n}} + 12 B^2 R^2 \sqrt{\frac{\log(2/\delta)}{2n}}. \]
The above can be reinterpreted as
\begin{gather*}
\sup_{g \in  \mathcal{F}_1 \cup \mathcal{F}_2} \abs{\norm{g}_{L^2(X, \rho)}^2 - \norm{g}_{L^2(X, \widehat{\rho})}^2} \leq 4BR^3 H_{max} + \frac{16 B^2 R^2}{\sqrt{n}} + 12 B^2 R^2 \sqrt{\frac{\log(2/\delta)}{2n}} \\
= 4BR^3 H_{max} + \tilde{O}\parens{\frac{B^2 R^2}{\sqrt{n}}}.
\end{gather*}
The desired result then follows from replacing $\delta$ with $\delta / 2$ in the previous argument.
\end{proof}
Now we are ready to provide a bound on the quantity $\norm{(T_K - T_n) r(\bullet; \theta)}_{L^2(X, \rho)}$ for $\theta$ satisfying $\norm{\theta - \theta_0}_2 \leq R$.
\begin{lem}\label{lem:longlemma}
Let $R \geq 1$ and let $B$ and $H_{max}$ be defined as in Lemma \ref{lem:combinedradbd}.  Let $\mathcal{C} = \{x \mapsto f_{lin}(x;\theta) - f^*(x) : \theta \in \overline{B}(\theta_0, R) \}$.
Then there are quantities $\Gamma$ and $\Phi$ such that
\[ \Gamma = \tilde{O}\parens{\frac{BR \sqrt{\log(\mathcal{N}(\mathcal{C}, L^2(X, \rho), \epsilon))}}{\sqrt{n}}} \]
and
\[\Phi = 4BR^3 H_{max} +  \tilde{O}\parens{\frac{B^2 R^2}{\sqrt{n}}} \]
such that with probability at least $1 - \delta$ over the sampling of $x_1, \ldots, x_n$
\[ \sup_{\theta \in \overline{B}(\theta_0, R)} \norm{(T_K - T_n) r(\bullet; \theta)}_{L^2(X, \rho)} \leq \Gamma + \kappa \brackets{\sqrt{R^4 H_{max}^2 + 2 \Phi} + \sqrt{4 \epsilon^2 + 2 \Phi}}. \]
\end{lem}
\begin{proof}
We will define $r_{lin}(x; \theta) = f_{lin}(x; \theta) - f^*(x)$.  Well then we have that
\begin{gather*}
\norm{(T_K - T_n) r(\bullet; \theta)}_{L^2(X, \rho)} \\
\leq \norm{(T_K - T_n) r_{lin}(\bullet; \theta)}_{L^2(X, \rho)} + \norm{(T_K - T_n)(f - f_{lin})(\bullet; \theta)}_{L^2(X, \rho)}.
\end{gather*}
Now let $E$ be a proper $\epsilon$-covering  of $\mathcal{C} = \{ r_{lin}(x; \theta) : \theta \in \overline{B}(\theta_0, R) \}$ with respect to $L^2(X, \rho)$.  Furthermore assume $E$ is of minimal cardinality so that $|E| = \mathcal{N}(\mathcal{C}, L^2(X, \rho), \epsilon)$.  Then for any $r_{lin}(\bullet; \theta)$ we can choose $\hat{\theta} \in \overline{B}(\theta_0, R)$ so that $r_{lin}(\bullet; \hat{\theta}) \in E$ and 
\[ \norm{r_{lin}(\bullet; \theta) - r_{lin}(\bullet; \hat{\theta})}_{L^2(X, \rho)} \leq \epsilon.\]
Well then
\begin{gather*}
\norm{(T_K - T_n) r_{lin}(\bullet; \theta)}_{L^2(X, \rho)} \\
\leq  \norm{(T_K - T_n) r_{lin}(\bullet; \hat{\theta})}_{L^2(X, \rho)} + \norm{(T_K - T_n)(r_{lin}(\bullet; \theta) - r_{lin}(\bullet; \hat{\theta}))}_{L^2(X, \rho)}.
\end{gather*}
We note that for any $r_{lin}(x; \theta) \in \mathcal{C}$ that
\begin{gather*}
|r_{lin}(x; \theta)| \leq |f_{lin}(x; \theta)| + |f^*(x)| = |\langle \nabla_{\theta} f(x; \theta_0), \theta - \theta_0 \rangle| + |f^*(x)| \\
\leq BR + \norm{f^*}_{L^\infty(X, \rho)} =: S.
\end{gather*}
To handle the term $\big\|(T_K - T_n)r_{lin}(\bullet; \hat{\theta})\big\|_{L^2(X, \rho)}$, for $g \in E$ we define the random variables $Z_i := g(x_i) K_{x_i} - \EE_{x \sim \rho}[g(x) K_x]$ taking values in the separable Hilbert space $\mathcal{H}$ where $\mathcal{H}$ is the RKHS associated with $K$.  We note that $(T_n - T_K)g$ is equal to $\frac{1}{n} \sum_{i = 1}^n Z_i$.  Well then note that $\norm{g(x) K_x}_{\mathcal{H}} = |g(x)| \norm{K_x}_{\mathcal{H}} \leq \norm{g}_{L^\infty(X, \rho)} \sqrt{K(x, x)} \leq S \kappa^{1/2}$ a.s.  Well then
\begin{gather*}
\norm{Z_i}_{\mathcal{H}} \leq \norm{g(x_i) K_{x_i}}_{\mathcal{H}} + \norm{\EE_{x \sim \rho}[g(x) K_x]}_{\mathcal{H}} \\ 
\leq S \kappa^{1/2} + \EE_{x \sim \rho} \norm{g(x) K_x}_{\mathcal{H}} \leq 2 S \kappa^{1/2}.
\end{gather*}
Then using Hoeffding's inequality for random variables taking values in a separable Hilbert space (see \citealt[Section 2.4]{rosasco10a}) we have
\begin{gather*}
\PP\parens{\norm{\frac{1}{n} \sum_{i = 1}^n Z_i}_\mathcal{H} > s} \leq 2 \exp\parens{ -n s^2 / 2 [2 S \kappa^{1/2}]^2} . 
\end{gather*}
Thus by the union bound and the fact that $\frac{1}{n} \sum_{i = 1}^n Z_i = (T_n - T_K)g$ we have that
\[ \PP\parens{\max_{g \in E} \norm{(T_n - T_K)g}_{\mathcal{H}} > s} \leq 2 |E| \exp\parens{ -n s^2 / 2 [2 S \kappa^{1/2}]^2}. \]
By setting
\[ s = \frac{2\sqrt{2} \cdot S \kappa^{1/2} \sqrt{\log \parens{\frac{2 |E|}{\delta}}}}{\sqrt{n}} = \tilde{O}\parens{\frac{BR \sqrt{\log(\mathcal{N}(\mathcal{C}, L^2(X, \rho), \epsilon))}}{\sqrt{n}}} \]
we get that with probability at least $1 - \delta$ over the sampling of $x_1, \ldots, x_n$
\[ \max_{g \in E} \norm{(T_n - T_K)g}_{\mathcal{H}} \leq s \]
and thus from the inequality $\norm{\bullet}_{L^2(X, \rho)} \leq \sqrt{\sigma_1} \norm{\bullet}_{\mathcal{H}}$ we get
\begin{equation}\label{eq:hoeffdingbd}
\max_{g \in E} \norm{(T_n - T_K)g}_{L^2(X, \rho)} \leq s \sqrt{\sigma_1} \leq s \sqrt{\kappa}.    
\end{equation}
On the other hand we must bound
\[ \norm{(T_K - T_n)(r_{lin}(\bullet; \theta) - r_{lin}(\bullet; \hat{\theta}))}_{L^2(X, \rho)} \]
and
\[ \norm{(T_K - T_n)(f - f_{lin})(\bullet; \theta)}_{L^2(X, \rho)}. \]
Well note since $K(\bullet, \bullet) \leq \kappa$ pointwise it follows by Cauchy-Schwarz that for any $h$
\[|T_K h(x)| = \abs{\int K(x, s) h(s) d\rho(s)} \leq \kappa \norm{h}_{L^2(X, \rho)} \]
and similarly
\[ |T_n h(x)| = \abs{\int K(x, s) h(s) d\widehat{\rho}(s)} \leq \kappa \norm{h}_{L^2(X, \widehat{\rho})}. \]
Therefore
\begin{gather*}
\norm{(T_K - T_n)h}_{L^2(X, \rho)} \leq \norm{(T_K - T_n)h}_{L^\infty(X, \rho)} \leq \norm{T_K h}_{L^\infty(X, \rho)} + \norm{T_n h}_{L^\infty(X, \rho)} \\
\leq \kappa [\norm{h}_{L^2(X, \rho)} + \norm{h}_{L^2(X, \widehat{\rho})}].    
\end{gather*}
Thus we will bound $r_{lin}(\bullet; \theta) - r_{lin}(\bullet; \hat{\theta})$ and $(f - f_{lin})(\bullet;\theta)$ in $L^2(X, \rho)$ and $L^2(X, \widehat{\rho})$.  Well since $\theta \in \overline{B}(\theta_0, R)$ we have that $(f - f_{lin})(\bullet; \theta) \in \mathcal{F}_1$ where $\mathcal{F}_1$ is defined as in Lemma \ref{lem:combinedradbd}.  On the other hand we note that $r_{lin}(x; \theta) - r_{lin}(x; \hat{\theta}) = \langle \nabla_\theta f(x; \theta_0), \theta - \hat{\theta} \rangle$.  Note that since $\theta, \hat{\theta} \in \overline{B}(\theta_0, R)$ we have that $\big\|\theta - \hat{\theta}\big\|_2 \leq 2R$.  Thus $r_{lin}(\bullet; \theta) - r_{lin}(\bullet;\hat{\theta}) \in \mathcal{F}_2$ where $\mathcal{F}_2$ is defined as in Lemma \ref{lem:combinedradbd}.  Thus by Lemma \ref{lem:combinedradbd} separate from the randomness before we have with probability at least $1 - \delta$ over the sampling of $x_1, \ldots, x_n$
\begin{equation}\label{eq:phibd}
\sup_{g \in  \mathcal{F}_1 \cup \mathcal{F}_2} \abs{\norm{g}_{L^2(X, \rho)}^2 - \norm{g}_{L^2(X, \widehat{\rho})}^2} \leq 4BR^3 H_{max} +  \tilde{O}\parens{\frac{B^2 R^2}{\sqrt{n}}} := \Phi.     
\end{equation}
Well note that by Lemma \ref{lem:fminusflinrad} we have that for each $i \in [n]$ 
\[|f(x_i; \theta) - f_{lin}(x_i; \theta)| \leq \frac{1}{2} R^2 H_{max} \]
and consequently
\[ \norm{f(\bullet; \theta) - f_{lin}(\bullet; \theta)}_{L^2(X, \widehat{\rho})} \leq \frac{1}{2} R^2 H_{max}. \]
On the other hand we had by the selection of $\hat{\theta}$ that
\[ \norm{r_{lin}(\bullet; \theta) - r_{lin}(\bullet; \hat{\theta})}_{L^2(X, \rho)} \leq \epsilon. \]
Now for conciseness let $h_1 = f(\bullet; \theta) - f_{lin}(\bullet; \theta)$ and $h_2 = r_{lin}(\bullet; \theta) - r_{lin}(\bullet; \hat{\theta})$.  Then by \eqref{eq:phibd} we have 
\[ \norm{h_1}_{L^2(X, \rho)}^2 \leq \norm{h_1}_{L^2(X, \widehat{\rho})}^2 + \Phi \leq \frac{1}{4}R^4 H_{max}^2 + \Phi \]
and
\[ \norm{h_2}_{L^2(X, \widehat{\rho})}^2 \leq \norm{h_2}_{L^2(X, \rho)}^2 + \Phi \leq \epsilon^2 + \Phi. \]
This implies
\begin{align*}
\norm{h_1}_{L^2(X, \rho)}^2 + \norm{h_1}_{L^2(X, \widehat{\rho})}^2 &\leq \frac{1}{2} R^4 H_{max}^2 + \Phi\\
\norm{h_2}_{L^2(X, \rho)}^2 + \norm{h_2}_{L^2(X, \widehat{\rho})}^2 &\leq 2 \epsilon^2 + \Phi.
\end{align*}
Thus using the inequality $a + b \leq \sqrt{2}(a^2 + b^2)^{1/2}$ for $a, b \geq 0$ combined with the previous estimates we have
\[ \norm{h_1}_{L^2(X, \rho)} + \norm{h_1}_{L^2(X,\widehat{\rho})} \leq \sqrt{2} \sqrt{\frac{1}{2} R^4 H_{max}^2 + \Phi} = \sqrt{R^4 H_{max}^2 + 2 \Phi} \]
and
\[ \norm{h_2}_{L^2(X, \rho)} + \norm{h_2}_{L^2(X,\widehat{\rho})} \leq \sqrt{2} \sqrt{2 \epsilon^2 + \Phi} = \sqrt{4 \epsilon^2 + 2 \Phi}. \]
Thus we have just shown that assuming \eqref{eq:phibd} holds that
\[ \norm{(T_K - T_n)h_1}_{L^2(X, \rho)} \leq \kappa [\norm{h_1}_{L^2(X, \rho)} + \norm{h_1}_{L^2(X, \widehat{\rho})}] \leq \kappa \sqrt{R^4 H_{max}^2 + 2 \Phi}  \]
and
\[ \norm{(T_K - T_n)h_2}_{L^2(X, \rho)} \leq \kappa [\norm{h_2}_{L^2(X, \rho)} + \norm{h_2}_{L^2(X, \widehat{\rho})}] \leq \kappa \sqrt{4 \epsilon^2 + 2 \Phi}. \]
Then by taking a union bound we can assume with probability at least $1 - 2\delta$ that \eqref{eq:hoeffdingbd} and \eqref{eq:phibd} hold simultaneously.  In which case our previous estimates combine to give us the bound
\begin{gather*}
\norm{(T_K - T_n)r(\bullet; \theta)}_{L^2(X, \rho)} \\
\leq \norm{(T_K - T_n)r_{lin}(\bullet;\hat{\theta})}_{L^2(X, \rho)} + \norm{(T_K - T_n)h_1}_{L^2(X, \rho)} + \norm{(T_K - T_n)h_2}_{L^2(X, \rho)} \\
\leq s\sqrt{\kappa} + \kappa \brackets{\sqrt{R^4 H_{max}^2 + 2 \Phi} + \sqrt{4 \epsilon^2 + 2 \Phi}}.
\end{gather*}
We now note that as long as \eqref{eq:hoeffdingbd} and \eqref{eq:phibd} hold the same argument runs through for any $\theta \in \overline{B}(\theta_0, R)$.  Thus with probability at least $1 - 2\delta$
\[ \sup_{\theta \in \overline{B}(\theta_0, R)} \norm{(T_K - T_n) r(\bullet; \theta)}_{L^2(X, \rho)} \leq s \sqrt{\kappa} + \kappa \brackets{\sqrt{R^4 H_{max}^2 + 2 \Phi} + \sqrt{4 \epsilon^2 + 2 \Phi}}. \]
The desired conclusion follows by setting $\Gamma = s \sqrt{\kappa}$ and replacing $\delta$ with $\delta / 2$ in the previous argument.
\end{proof}
From Lemma~\ref{lem:longlemma} we get the following corollary.
\begin{cor}\label{cor:tkminustn}
Let $R \geq 1$ and
\[B = \sup_{x \in X} \sup_{\theta \in \overline{B}(\theta_0, R)} \norm{\nabla_\theta f(x, \theta)}_2.\]
Then with probability at least $1 - \delta$ over the sampling of $x_1, \ldots, x_n$ we have that
\[\sup_{\theta \in \overline{B}(\theta_0, R)} \norm{(T_K - T_n) r(\bullet; \theta)}_{L^2(X, \rho)}^2 \leq \epsilon \]
provided that
$B = \tilde{O}(1)$, $H_{max} = \tilde{O}(\epsilon / R^3)$ and $n = \tilde{\Omega}(R^4 / \epsilon^2)$ where the expressions under the $\tilde{O}$ and $\tilde{\Omega}$ notation do not depend on the values $x_1, \ldots, x_n$.
\end{cor}
\begin{proof}
After substituting $\epsilon^{1/2}$ for $\epsilon$ in Lemma \ref{lem:longlemma} we have that with probability at least $1 - \delta$ over the sampling of $x_1, \ldots, x_n$
\[ \sup_{\theta \in \overline{B}(\theta_0, R)} \norm{(T_K - T_n) r(\bullet; \theta)}_{L^2(X, \rho)} \leq \Gamma + \kappa \brackets{\sqrt{R^4 H_{max}^2 + 2 \Phi} + \sqrt{4 \epsilon + 2 \Phi}} \]
where
\[ \Gamma = \tilde{O}\parens{\frac{BR \sqrt{\log(\mathcal{N}(\mathcal{C}, L^2(X, \rho), \epsilon^{1/2}))}}{\sqrt{n}}}, \]
\[\Phi = 4BR^3 H_{max} +  \tilde{O}\parens{\frac{B^2 R^2}{\sqrt{n}}}, \]
and
\[\mathcal{C} = \{x \mapsto f_{lin}(x; \theta) - f^*(x) : \theta \in \overline{B}(\theta_0, R) \}. \]
Now define
\[ F := \int_X \nabla_\theta f(x; \theta_0) \nabla_\theta f(x; \theta_0)^T d\rho(x). \]
Since translation by a fixed function does not change the covering number we have by Corollary~\ref{cor:covnum} that
\[\log \mathcal{N}(\mathcal{C}, L^2(X, \rho), \epsilon^{1/2}) = \tilde{O}\parens{\tilde{p}\parens{F^{1/2}\, \frac{3\epsilon^{1/2}}{4R}}} = \tilde{O}\parens{\tilde{p}\parens{F, \frac{9\epsilon}{16R^2}}}. \]
Well using the fact that $\tilde{p}(A, \epsilon) \leq \frac{Tr(A)}{\epsilon}$ we have that
\[ \tilde{p}\parens{F, \frac{9\epsilon}{16R^2}} \leq \frac{16R^2 Tr(F)}{9 \epsilon}. \]
Well we note that
\begin{gather*}
Tr(F) = Tr(\EE_{x \sim \rho}[\nabla_\theta f(x; \theta_0) \nabla_\theta f(x; \theta_0)^T]) = \EE_{x \sim \rho} Tr(\nabla_\theta f(x; \theta_0) \nabla_\theta f(x; \theta_0)^T) 
\\ = \EE_{x \sim \rho} \norm{\nabla_\theta f(x; \theta_0)}^2 \leq B^2.
\end{gather*}
Therefore assuming $B = \tilde{O}(1)$ we have that
\[ \Gamma = \tilde{O}\parens{\frac{R \sqrt{\log \mathcal{N}(\mathcal{C}, L^2(X, \rho), \epsilon^{1/2})}}{\sqrt{n}}} = \tilde{O}\parens{\frac{R^2}{\epsilon^{1/2} \sqrt{n}}}. \]
Thus $n = \tilde{\Omega}(R^4 / \epsilon^2)$ suffices to ensure that $\Gamma = O(\epsilon^{1/2})$.  Now we must bound
\[\Phi = 4BR^3 H_{max} +  \tilde{O}\parens{\frac{B^2 R^2}{\sqrt{n}}}. \]
We note that whenever $B = \tilde{O}(1)$ we have that
$H_{max} = \tilde{O}(\epsilon / R^3)$ and $n = \tilde{\Omega}(R^4 / \epsilon^2)$ guarantees that $\Phi = O(\epsilon)$.  Finally we have that $H_{max} = \tilde{O}(\epsilon / R^3) \subset \tilde{O}(\epsilon^{1/2} / R^2)$ suffices to ensure that $R^4 H_{max}^2 = O(\epsilon)$.  Thus given all these conditions are met we have that
\[ \Gamma + \kappa \brackets{\sqrt{R^4 H_{max}^2 + 2 \Phi} + \sqrt{4 \epsilon + 2 \Phi}} = O(\epsilon^{1/2}). \]
The desired result then follows from setting the constants under the $\tilde{O}$ and $\tilde{\Omega}$ notation appropriately.
\end{proof}
The following lemma combines the results in this section to get the ultimate bound on the operator deviations $T_K - T_n^t$.
\begin{lem}\label{lem:bigbadlemma}
Let $R \geq 1$ and $\epsilon \in (0, R)$.  Let $S = (x_1, \ldots, x_n)$ and $S' = (x_1', \ldots, x_n')$ be two separate i.i.d.\ samples from $\rho$ and denote 
\[ H_{max} := \max_{z \in S \cup S'} \sup_{\theta \in \overline{B}(\theta_0, R)} \norm{H(z, \theta)}_{op} \]
\[B := \sup_{x \in X} \sup_{\theta \in \overline{B}(\theta_0, R)} \norm{\nabla_\theta f(x, \theta)}_2.\]
Then with probability at least $1 - \delta$ over the sampling of $S$, $S'$ we have that for any $t$ such that $\norm{\theta_t - \theta_0}_2 \leq R$ that
\[ \norm{(T_K - T_n^t) r_t}_{L^2(X, \rho)}^2  \leq 4 \norm{f^*}_{L^\infty(X, \rho)}^2 \norm{K - K_0}_{L^2(X^2, \rho \otimes \rho)}^2 + \epsilon \]
provided that $B = \tilde{O}(1)$, $H_{max} = \tilde{O}(\epsilon / R^3)$ and $n = \tilde{\Omega}(R^4 / \epsilon^2)$ where the expressions under the $\tilde{O}$ and $\tilde{\Omega}$ notation do not depend on $S$ and $S'$.
\end{lem}
\begin{proof}
We note that for $\theta_t$ such that $\norm{\theta_t - \theta_0}_2 \leq R$ that
\begin{align*}
\norm{(T_K - T_n^t) r_t}_{L^2(X, \rho)}^2 &\leq [\norm{(T_K - T_n) r_t}_{L^2(X, \rho)} + \norm{(T_n - T_n^t) r_t}_{L^2(X, \rho)}]^2 \\
&\leq 2 \norm{(T_K - T_n) r_t}_{L^2(X, \rho)}^2 + 2\norm{(T_n - T_n^t) r_t}_{L^2(X, \rho)}^2\\
&\leq 2 \sup_{\theta \in \overline{B}(\theta_0, R)} \norm{(T_K - T_n) r(\bullet; \theta)}_{L^2(X, \rho)}^2 + 2 \norm{(T_n - T_n^t) r_t}_{L^2(X, \rho)}^2.
\end{align*}
Well by Corollary \ref{cor:tkminustn} we have with probability at least $1 - \delta$ over the sampling of $x_1, \ldots, x_n$
\[\sup_{\theta \in \overline{B}(\theta_0, R)} \norm{(T_K - T_n) r(\bullet; \theta)}_{L^2(X, \rho)}^2 \leq \epsilon \]
provided that $B = \tilde{O}(1)$,
$H_{max} = \tilde{O}(\epsilon / R^3)$ and $n = \tilde{\Omega}(R^4 / \epsilon^2)$.  This result also does not depend in any way on $S'$.  On the other hand by Corollary \ref{cor:tnminustnt} separate from the randomness before we have with probability at least $1 - \delta$ over the sampling of $S$ and $S'$ that for any $\theta_t$ such that $\norm{\theta_t - \theta_0}_2 \leq R$
\begin{gather*}
\norm{(T_n - T_n^t) r_t}_{L^2(X, \rho)}^2 \leq 2 \norm{f^*}_{L^\infty(X, \rho)}^2 \norm{K - K_0}_{L^2(X^2, \rho \otimes \rho)}^2 + \epsilon.
\end{gather*}
provided that $B = \tilde{O}(1)$, $H_{max} = \tilde{O}(\epsilon^{1/2} / R)$ and $n = \tilde{\Omega}(\epsilon^{-2}).$  The desired result then follows from taking a union bound and replacing $\delta$ with $\delta / 2$ and $\epsilon$ with $\epsilon / 4$. 
\end{proof}

\section{Main Result}
\label{sec:main_result_proof}
\subsection{Damped Deviations}\label{sec:dampeddev}
In this subsection we will recall some definitions and results from \citet{bowman2022implicit}.  The main theorems in \citet{bowman2022implicit} assume that the network architecture is shallow, however the results we recall in this section do not depend on the architecture.  Let $K(x, x')$ be a continuous, symmetric, positive-definite kernel.  Recall that $K$ defines the integral operator
\[ T_K g(x) := \int_X K(x, s) g(s) d\rho(s). \]
Then by Mercer's theorem
\[ K(x, x') = \sum_{i = 1}^\infty \sigma_i \phi_i(x) \phi_i(x') \]
where $\{\phi_i\}_i$ is an orthonormal basis of $L^2(X, \rho)$ and $\{\sigma_i\}_i$ is a nonincreasing sequence of positive values.  Each $\phi_i$ is an eigenfunction of $T_K$ with eigenvalue $\sigma_i$, i.e. $T_K \phi_i = \sigma_i \phi_i$.  Let $x \mapsto g_s(x)$ be a $L^2(X, \rho)$ function for each $s \in [0, t]$.  Assume $s \mapsto \langle \phi_i, g_s \rangle_\rho$ is measureable for each $i$ and $\int_0^t \norm{g_s}_{L^2(X, \rho)}^2 ds < \infty$.  Then we write
\[ \int_0^t g_s ds \]
to denote the coordinate-wise integral, meaning that $\int_0^t g_s ds$ is the $L^2(X, \rho)$ function $h$ such that
\[ \langle h, \phi_i \rangle_\rho = \int_0^t \langle g_s, \phi_i \rangle_\rho ds. \]
With this definition in hand we now recall the following ``Damped Deviations'' lemma given by \citet[Lemma 2.4]{bowman2022implicit}. 
\begin{lem}\label{lem:dampeddevfunc}
Let $K(x, x')$ be a continuous, symmetric, positive-definite kernel.  Let $[T_{K} h](\bullet) = \int_X K(\bullet, s) h(s) d\rho(s)$ be the integral operator associated with $K$ and let $[T_n^s h](\bullet) = \frac{1}{n} \sum_{i = 1}^n K_s(\bullet, x_i) h(x_i)$ denote the operator associated with the time-dependent NTK $K_s$.  Then
\[r_t = \exp(-T_{K}t)r_0 + \int_0^t \exp(-T_{K}(t - s))(T_{K} - T_n^s)r_s ds , 
\]
where the equality is in the $L^2(X, \rho)$ sense.
\end{lem}
Furthermore we have the following lemma \citet[Lemma C.8]{bowman2022implicit}
\begin{lem}\label{lem:kerdiffrecipe}
Let $K(x, x')$ be a continuous, symmetric, positive-definite kernel with associated operator $T_K h(\bullet) = \int_X K(\bullet, s) h(s) d\rho(s)$.  Let $T_n^s h(\bullet) = \frac{1}{n} \sum_{i = 1}^n K_s(\bullet, x_i) h(x_i)$ denote the operator associated with the time-dependent NTK.  Then
\[\norm{P_k(r_t - \exp(-T_{K}t)r_0)}_{L^2(X, \rho)} \leq \frac{1 - \exp(-\sigma_k t)}{\sigma_k} \sup_{s \leq t} \norm{(T_{K} - T_n^s)r_s}_{L^2(X, \rho)}.\]
and
\[\norm{r_t - \exp(-T_{K}t)r_0}_{L^2(X, \rho)} \leq t \cdot \sup_{s \leq t} \norm{(T_{K} - T_n^s)r_s}_{L^2(X, \rho)}.\]
\end{lem}
\subsection{Proof of Theorem~\ref{thm:main}}
We are now ready to prove the main result of this paper.
\main*
\begin{proof}
Let $\theta_0$ be the parameter initialization and let $S = (x_1, \ldots, x_n)$ and $S' = (x_1', \ldots, x_n')$ be two i.i.d.\ samples from $\rho$.  Furthermore let $1 \leq R \leq \sqrt{m}$.  Let $E_1 \subset \RR^p \times X^{2n}$ be the set of values $(\theta_0, S, S')$ so that the conclusion of Lemma \ref{lem:bigbadlemma} holds.  Similarly let $E_2$ be the set of values $(\theta_0, S, S')$ satisfying
\[ B := \max_{x \in X} \sup_{\theta \in \overline{B}(\theta_0, R)} \norm{\nabla_\theta f(x; \theta)}_2 = O(1) \]
and
\[ H_{max} := \max_{z \in S \cup S'} \sup_{\theta \in \overline{B}(\theta_0, R)} \norm{H(z, \theta)}_{op} = \tilde{O}(\epsilon / R^3) \]
where the expression $O(1)$ above is the bound on $B$ given by Lemma~\ref{lem:gradbd} and the expression $\tilde{O}(\epsilon/R^3)$ is precisely the condition on $H_{max}$ in the conclusion of Lemma \ref{lem:bigbadlemma}.  By Lemma \ref{lem:bigbadlemma} for any fixed $\theta_0$ we have that the conclusion holds with probability at least $1 - \delta$ over the sampling of $S, S'$.  Thus for any $\theta_0$ we have that
\[\EE_{S,S'}[\ind{(\theta_0, S, S') \in E_1}] \geq 1 - \delta. \]
It follows then by the Fubini-Tonelli theorem that 
\[\PP(E_1) = \EE_{\theta_0} \EE_{S,S'}[\ind{(\theta_0, S, S') \in E_1}] \geq 1 - \delta. \]
On the other hand by Theorem \ref{thm:liuhessbd} and Lemma \ref{lem:gradbd} combined with a union bound we have that for any fixed $S, S'$ then with probability at least $1 -2Cmn\exp(-c\log^2(m)) - C\exp(-cm)$ that $H_{max} = \tilde{O}(R / \sqrt{m})$ and $B = O(1)$.  Thus if $m = \tilde{\Omega}(R^8 / \epsilon^2)$ we ensure that $H_{max} = \tilde{O}(\epsilon / R^3)$.  Then by the same Fubini-Tonelli argument as before we get that
\[ \PP(E_2) = \EE_{S, S'} \EE_{\theta_0} \ind{(\theta_0, S, S') \in E_2} \geq 1 - 2Cmn \exp(-c\log^2(m)) - 
C\exp(-c m). \]
Thus by taking a union bound we have with probability at least $1 - \delta - O(mn)\exp(-\Omega(\log^2(m))$ that the events $E_1$ and $E_2$ both hold simultaneously.  This holds for any $\delta$ so we may as well set $\delta = O(mn) \exp(-\Omega(\log^2(m)))$ and absorb it into the other term.  Whenever $E_1$ and $E_2$ hold simultaneously we have by Lemma \ref{lem:bigbadlemma} that for any $\theta_t$ such that $\norm{\theta_t - \theta_0}_2 \leq R$
\begin{equation}\label{eq:swaelee}
\norm{(T_K - T_n^t)r_t}_{L^2(X, \rho)}^2 \leq 4 \norm{f^*}_{L^\infty(X, \rho)}^2 \norm{K - K_0}_{L^2(X^2, \rho \otimes \rho)}^2 + \epsilon . 
\end{equation}
Well by Lemma \ref{lem:aprioriparambd} we have that $\norm{\theta_t - \theta_0} \leq \frac{\sqrt{t}}{\sqrt{2}} \norm{f^*}_{L^\infty(X, \rho)}$.  Thus for $t \leq \frac{2R^2}{\norm{f^*}_{L^\infty(X, \rho)}^2}$ we have that $\norm{\theta_t - \theta_0} \leq R$.  Well then by Lemma \ref{lem:kerdiffrecipe} and the inequality \eqref{eq:swaelee} we have that
\begin{gather*}
\norm{P_k(r_t - \exp(-T_{K}t)r_0)}_{L^2(X, \rho)}^2 \\
\leq \brackets{\frac{1 - \exp(-\sigma_k t)}{\sigma_k}}^2 \cdot \brackets{ 4 \norm{f^*}_{L^\infty(X, \rho)}^2 \norm{K - K_0}_{L^2(X^2, \rho \otimes \rho)}^2 + \epsilon}
\end{gather*}
and
\[\norm{r_t - \exp(-T_{K}t)r_0}_{L^2(X, \rho)}^2 \leq t^2 \cdot \brackets{4 \norm{f^*}_{L^\infty(X, \rho)}^2 \norm{K - K_0}_{L^2(X^2, \rho \otimes \rho)}^2 + \epsilon}.\]
The desired result then follows by setting $T = \frac{2R^2}{\norm{f^*}_{L^\infty(X, \rho)}^2}$.
\end{proof}

\section{Discussion of Assumption~\ref{ass:kernelconc}}\label{sec:assumptiondisc}
We will discuss why it is reasonable to assume that $m = \tilde{\Omega}(\epsilon^{-2})$ suffices to ensure that $\norm{K_0 - K^{\infty}}_{L^2(X \times X, \rho \otimes \rho)}^2 \leq \epsilon$ holds with high probability over the initialization.  We note that for fixed $\theta_0$, $K_0$ and $K^\infty$ are bounded and thus by Hoeffding's inequality we have that with high probability
\begin{gather*}
\norm{K_0 - K^{\infty}}_{L^2(X \times X, \rho \otimes \rho)}^2 \\
\leq \frac{1}{N} \sum_{i = 1}^N |K_0(x_i, x_i') - K^\infty(x_i, x_i')|^2 + \tilde{O}\parens{\frac{\norm{K_0 - K^\infty}_{L^\infty(X \times X, \rho \times \rho)}^2}{\sqrt{N}}},     
\end{gather*}
where $(x_1, x_1'), \ldots, (x_N, x_N')$ is an i.i.d.\ sample from $\rho \otimes \rho$.  Furthermore we have by Lemma~\ref{lem:gradbd} that $\norm{K_0 - K^\infty}_{L^\infty(X \times X, \rho \times \rho)}^2 = \tilde{O}(1)$ with high probability over the initialization of $\theta_0$.  Thus if we set $N = \tilde{\Omega}(\epsilon^{-2})$ we have that Assumption~\ref{ass:kernelconc} holds provided that
\[ \frac{1}{N} \sum_{i = 1}^N |K_0(x_i, x_i') - K^\infty(x_i, x_i')|^2 = O(\epsilon) \]
with high probability over the simultaneous sampling of $\theta_0$ and $(x_1, x_1'), \ldots, (x_N, x_N')$.  
\par
It is been shown in many settings that the pointwise deviations satisfy 
\[ |K_0(x, x') - K^\infty(x, x')| = \tilde{O}(1/\sqrt{m}) \] with high probability over $\theta_0$.  The earliest was \citet{du2018gradient} who demonstrate that for a shallow ReLU network for fixed $x, x'$ we have with probability at least $1 - \delta$ over the initialization
\[ |K_0(x, x') - K^\infty(x, x')| \leq O\parens{\frac{\log(1/\delta)}{\sqrt{m}}}. \]
Analyzing the portion of the Neural Tangent Kernel corresponding to the last hidden layer, \citet{du2019gradient} get an analogous bound for deep fully-connected, ResNet, and convolutional networks with smooth activations.  This is substantiated by the results of \citet{huang2019dynamics} for deep fully-connected networks with smooth activations.  In their work they demonstrate that for a fixed training set $x_1, \ldots, x_n$
\[ \max_{i, j} |K_0(x_i, x_j) - K^\infty(x_i, x_j)| = \tilde{O}(1/\sqrt{m}) \]
with high probability over the initialization.  In their result there are constants that depend on how well dispersed $x_1, \ldots, x_n$ are.  \citet{bowman2022implicit} demonstrated that for shallow fully-connected networks with smooth activations
\[ \sup_{(x, x') \in X \times X} |K_0(x, x') - K^\infty(x, x')| = \tilde{O}(1/\sqrt{m}) \]
with high probability over the initialization.  For deep fully-connected ReLU networks \citet{aroraexact} demonstrate that for fixed $x, x'$ if $m = \Omega(L^6 \log(L / \delta) / \epsilon^4)$ then with probability at least $1 - \delta$
\[ |K_0(x, x') - K^\infty(x, x')| \leq (L + 1) \epsilon.  \]
In terms of the width $m$ this translates to $|K_0(x, x') - K^\infty(x, x')| = \tilde{O}(1/m^{1/4})$ with high probability.  This was improved in a recent work by \citet{buchanan2021deep} that demonstrated that if $\mathcal{M}$ is a Riemannian submanifold of the unit sphere then with high probability over the initialization
\[ \sup_{x, x' \in \mathcal{M} \times \mathcal{M}}|K_0(x, x') - K^\infty(x, x')| = \tilde{O}(1/\sqrt{m}).  \]
Furthermore as stated by \citet{buchanan2021deep} their analysis should be amenable to other architectures.
\par 
Now note that $\max_{i \in [N]}|K_0(x_i, x_i') - K^\infty(x_i, x_i')| = O(\epsilon^{1/2})$ suffices to ensure that
\[ \frac{1}{N} \sum_{i = 1}^N |K_0(x_i, x_i') - K^\infty(x_i, x_i')|^2 = O(\epsilon). \]
Based on the previous discussion, we expect that with high probability 
\[\max_{i \in [N]}|K_0(x_i, x_i') - K^\infty(x_i, x_i')| = \tilde{O}(1/\sqrt{m}). \]
Thus if $m = \tilde{\Omega}(1/\epsilon^2)$ then we would have that 
$\max_{i \in [N]}|K_0(x_i, x_i') - K^\infty(x_i, x_i')| = \tilde{O}(\epsilon)$ which is stronger than what we need.  In fact $\max_{i \in [N]}|K_0(x_i, x_i') - K^\infty(x_i, x_i')| = \tilde{O}(1/ m^{1/4})$ is sufficient.  For these reasons, we view Assumption~\ref{ass:kernelconc} as quite reasonable.  Nevertheless, we are not aware of an out-of-the box result that simultaneously addresses all the cases we consider and thus we must add this as an external assumption.  However, if desired one can bypass Assumption~\ref{ass:kernelconc} by citing the aforementioned results to get statements for the cases in which they apply to.
\section{Experimental Details}\label{sec:experimental}

\paragraph{Architecture and Parameterization}
The code to produce Figure~\ref{fig:ntk_spec} is available at
\ifx\deanonymize\undefined
\url{https://anonymous.4open.science/r/deepspec-5564/README.md}
\else
\url{https://github.com/bbowman223/deepspec}
\fi  The NTK Gram matrix $(G_0)_{i,j} := K^{\theta_0}(x_i, x_j) = \langle \nabla_\theta f(x_i; \theta_0), \nabla_\theta f(x_j; \theta_0) \rangle$ was computed for two separate networks.  The first network corresponds to LeNet-5 \citep{lenetpaper} where the output is the logit corresponding to class $0$.  The second network is a feedforward network with one hidden layer with the Softplus activation $\omega(x) = \log(1 + \exp(x))$.  For LeNet-5 we compute the NTK using PyTorch \citep{pytorch} using the default PyTorch initialization and parameterization.  For the shallow network we implement the network directly and use the Neural Tangent Kernel parameterization:
\[ f(x; \theta) = \frac{1}{\sqrt{m}} \sum_{i = 1}^m a_i \omega(\langle w_i, x \rangle + b_i) + b_0, \]
where there is an explicit $1/\sqrt{m}$ factor.  All parameters for the shallow network are initialized as i.i.d.\ standard Gaussian random variables $N(0, 1)$. 

\paragraph{Details of Computation}
For each network we compute the NTK Gram matrix $G_0$ for 10 separate pairs of $(\theta_0, S)$ where $\theta_0$ is the parameter initialization and $S = (x_1, \ldots, x_n)$ is the data batch.  Each line in the plots of Figure~\ref{fig:ntk_spec} corresponds to a different pair $(\theta_0, S)$.  We simultaneously sample the parameter initialization $\theta_0$ and a random batch of $2000$ training samples $x_1, \ldots, x_{2000}$.  We load the batches using ``DataLoader'' in PyTorch with the ``shuffle'' parameter set to True.  This means the batches will be sampled sequentially from a random permutation of the training data and thus are sampled without replacement.  We then compute the NTK Gram matrix $(G_0)_{i,j} := K^{\theta_0}(x_i, x_j) = \langle \nabla_\theta f(x_i; \theta_0), \nabla_\theta f(x_j; \theta_0) \rangle$.  Once we compute $G_0$ we compute its spectrum and plot the first $1000$ eigenvalues.  Note that the number of eigenvalues that we plot is half the batch size.  We observe that if one plots all $n$ eigenvalues (the number of eigenvalues equals the number of samples) one gets a sharp drop in log scale magnitude starting near the bottom 5-10\% of eigenvalues.  We observed this to occur even as one varies $n$.  We suspect this is due to numerical errors and thus we only plot the first half of the spectrum. 

\paragraph{Data}
The dataset used for LeNet-5 is MNIST \citep{lenetpaper} and the dataset for the shallow model is CIFAR-10 \citep{Krizhevsky09learningmultiple}.  MNIST is made available through the Creative Commons Attribution-Share Alike 3.0 license.  CIFAR-10 does not specify a license.  Neither of these datasets have personally identifiable information nor offensive content.

\paragraph{Computational Resources and Runtime}
The experiments were run on a 2016 Macbook Pro with a 2.6 Ghz Quad-Core Intel Core i7 processor and 16GB of RAM.  The experiment took less than an hour in wall-clock time. 

\paragraph{Software Licenses and Attribution}
Our experiments were implemented in Python with the aid of the following software libraries/tools: PyTorch \citep{pytorch}, NumPy \citep{harris2020array}, SciPy \citep{2020SciPy-NMeth}, Matplotlib \citep{Hunter:2007}, Jupyter Notebook \citep{Kluyver2016jupyter}, IPython \citep{ipython}, and autograd-hacks \url{https://github.com/cybertronai/autograd-hacks}.  PyTorch, Numpy, and SciPy are available under the BSD license.  Jupyter and IPython are available under the new/modified BSD license.  Matplotlib uses only BSD compatible code and is available under the PSF license.  The code for autograd-hacks belongs to the public domain as specified by the public-domain-equivalent-license ``Unlicense'' \url{https://unlicense.org/}.

\end{document}